\documentclass{article}

\usepackage[main, final]{neurips_2026}
\setcitestyle{numbers,square,comma}

\usepackage[utf8]{inputenc}
\usepackage[T1]{fontenc}
\usepackage{hyperref}
\hypersetup{colorlinks=true, linkcolor=black, citecolor=black, urlcolor=black}
\usepackage{url}
\usepackage{booktabs}
\usepackage{amsfonts}
\usepackage{amssymb}
\usepackage{amsmath}
\usepackage{mathtools}
\usepackage{amsthm}
\usepackage{bm}
\usepackage{nicefrac}
\usepackage{microtype}
\usepackage[pdftex]{graphicx}
\usepackage{subcaption}
\usepackage[table]{xcolor}
\usepackage{multirow}
\usepackage{array}
\usepackage{adjustbox}
\usepackage{algorithm}
\usepackage{algpseudocode}
\usepackage{float}
\usepackage[inline]{enumitem}
\definecolor{oursred}{RGB}{250,235,235}
\definecolor{pgcblue}{RGB}{235,245,255}

\definecolor{darkblue}{rgb}{0.0, 0.0, 0.55}

\usepackage{hyperref}
\hypersetup{
    colorlinks=true,
    linkcolor=darkblue,
    citecolor=darkblue,
    urlcolor=darkblue,
}

\usepackage[most]{tcolorbox}
\usepackage{xcolor}

\definecolor{creditbg}{RGB}{232,244,255}
\definecolor{creditframe}{RGB}{55,120,190}
\definecolor{pathbg}{RGB}{255,241,224}
\definecolor{pathframe}{RGB}{205,120,45}
\definecolor{principlebg}{RGB}{248,248,248}
\definecolor{principleframe}{RGB}{120,120,120}

\newtcbox{\creditbox}{
  nobeforeafter,
  math upper,
  tcbox raise base,
  enhanced,
  colback=creditbg,
  colframe=creditframe,
  boxrule=0.5pt,
  arc=2pt,
  left=2pt,
  right=2pt,
  top=1pt,
  bottom=1pt
}

\newtcbox{\pathbox}{
  nobeforeafter,
  math upper,
  tcbox raise base,
  enhanced,
  colback=pathbg,
  colframe=pathframe,
  boxrule=0.5pt,
  arc=2pt,
  left=2pt,
  right=2pt,
  top=1pt,
  bottom=1pt
}

\newenvironment{tightdisplay}
{
  \begingroup
  \setlength{\abovedisplayskip}{2pt}
  \setlength{\belowdisplayskip}{2pt}
  \setlength{\abovedisplayshortskip}{2pt}
  \setlength{\belowdisplayshortskip}{2pt}
}
{
  \endgroup
}

\theoremstyle{plain}
\newtheorem{theorem}{Theorem}[section]
\newtheorem{proposition}[theorem]{Proposition}
\newtheorem{lemma}[theorem]{Lemma}

\theoremstyle{definition}
\newtheorem{definition}[theorem]{Definition}
\newtheorem{assumption}[theorem]{Assumption}
\theoremstyle{remark}

\newcommand{\R}{\mathbb{R}}
\newcommand{\E}{\mathbb{E}}
\newcommand{\cP}{\mathcal{P}}
\newcommand{\cA}{\mathcal{A}}

\newcommand{\cF}{\mathcal{F}}
\newcommand{\cK}{\mathcal{K}}
\newcommand{\vx}{\bm{x}}
\newcommand{\vz}{\bm{z}}

\newcommand{\vv}{\bm{v}}
\newcommand{\vg}{\bm{g}}

\newcommand{\vphi}{\bm{\phi}}
\newcommand{\vtheta}{\bm{\theta}}
\newcommand{\eps}{\varepsilon}
\newcommand{\dd}{\,\mathrm{d}}

\title{From Baselines to Transport Geodesics: Axiomatic Attribution via Optimal Generative Flows}

\author{%
  Cenwei Zhang\thanks{Equal contribution.} \textsuperscript{} \thanks{Work done while at Shanghai Jiao Tong University.} \\
  The Chinese University of Hong Kong\\
  Hong Kong, China\\
  \texttt{cwzhang2001@gmail.com}
  \And
  Lin Zhu\footnotemark[1] \textsuperscript{} \thanks{Work done while at Technical University of Denmark.} \\
  Aalto University\\
  Espoo, Finland\\
  \texttt{lin.1.zhu@aalto.fi}
  \And
  Manxi Lin \\
  Alibaba\\
  Hangzhou, China\\
  \texttt{linmanxi.lmx@alibaba-inc.com}
  \And
  Lei You\thanks{Corresponding author.} \\
  Technical University of Denmark\\
  Copenhagen, Denmark\\
  \texttt{leiyo@dtu.dk} \\
}

\begin{document}

\maketitle

\begin{abstract}
Feature attributions often hide a critical modeling choice: 
they explain a prediction along a counterfactual path from a reference state to an input.
Different baselines, 
interpolations, 
and generative trajectories define different paths and can therefore produce different explanations. 
We study this path ambiguity as a modeling problem. 
Our central question is whether the path can be chosen by a distribution-level transport process, rather than by a hand-designed interpolation or by the sensitivity geometry of the model being explained. We separate attribution into fixed-path credit allocation and path selection. 
For a fixed path, we prove that the Aumann–Shapley line integral is the unique attribution rule under fixed-path axioms and coordinate-trace separability. 
For path selection, we minimize kinetic
action over flows that transport a reference distribution to the data
distribution, yielding a transport-geodesic attribution principle. 
We instantiate this framework using Rectified Flow and derive stability bounds quantifying how vector-field approximation errors affect the resulting attributions.
Experiments show that Reflow improves cross-seed attribution stability, while learned transport paths yield spatially coherent explanations with competitive deletion faithfulness among the evaluated path-based methods.
Our code is available at \url{https://github.com/cenweizhang/OTFlowSHAP}.
\end{abstract}

\section{Introduction}
\label{sec:introduction}

Feature attribution asks a simple question: which input coordinates made a model score go up or down for a particular example? This question is simple only at the surface. To explain an image classifier, one usually compares the observed image with a reference state where some information is absent. A black image, a blurred image, a mean image, a masked image, and an inpainted image all define different meanings of absence. They also define different transitions from absence to presence. Since a neural network is nonlinear, the contribution assigned to a pixel or region can change when this transition changes. We therefore view the main ambiguity in attribution as a \emph{path ambiguity}. The problem is not only which attribution formula one should use, but also which counterfactual path one should trust.

Classical Shapley values give a principled answer for finite cooperative games \citep{shapley1953value}. They allocate the total payoff according to efficiency, symmetry, dummy, and additivity. Model explanation methods such as SHAP and KernelSHAP adapt this idea to predictors by treating features as players \citep{strumbelj2010efficient,lundberg2017unified}. This adaptation is powerful, but it also exposes the absence problem. To evaluate a coalition, one must define the model value when other features are missing. Different choices give different Shapley games and different explanations \citep{sundararajan2020many}. In images, many such counterfactuals look unnatural or low density, so the model may be queried in regions that it never learned to handle.

Path-integral methods reduce the combinatorial burden. Integrated Gradients connects a baseline to the input by a straight line and integrates the gradient along that path \citep{sundararajan2017axiomatic}. Expected Gradients averages this idea over references \citep{erion2021improving}. Guided or blurred paths modify the interpolation to reduce visual noise \citep{kapishnikov2021guided}. Recent geometric variants go further by replacing the straight line with Riemannian geodesics on a learned data geometry or on a model-induced geometry \citep{zaher2024manifold,salek2025using}. These works make clear that path choice matters. We do not try to show that a transport path is universally stronger than these geometric paths. Explanation methods are difficult to rank by a single number because deletion, insertion, stability, localization, and visual structure measure different properties. Our aim is different, namely, we ask:
\vspace{-2.5mm}
\begin{quote}
\emph{Whether the attribution path can be defined by the distribution-level transport process rather than by the sensitivity geometry of the model being explained?}
\end{quote}
\vspace{-2.5mm}
In other words, we study a different path-selection object:
not a hand-designed instance-level curve and not a model-induced Riemannian
geodesic, but a least-action generative transport process from a reference
distribution to the data distribution.

This viewpoint changes the role of the generative model. A diffusion or flow model can certainly provide a structured trajectory, but a structured trajectory is not automatically a principled explanation path. We use optimal transport to specify the target path-selection principle. Among all flows that move a reference distribution \(p_0\) to the data distribution \(p_1\), we choose the one with minimum kinetic action. By the dynamic formulation of optimal transport, this is the Wasserstein-2 geodesic in distribution space \citep{benamou2000computational,villani2008optimal}. A sample explanation then follows a characteristic curve of this least-action flow and integrates the predictor gradient along that curve.

We separate the construction into two parts. The first part is allocation along a fixed path. Once a smooth path \(\gamma\) from a reference state to the input is given, we ask which rule can split the score change \(f(\gamma(1))-f(\gamma(0))\) among coordinates. We show that natural fixed-path axioms, together with coordinate-trace separability, force the Aumann-Shapley
line integral. This result says that the allocation rule is not the place where one should add more heuristics. The second part is path selection for which we adopt the least-action optimal transport principle. 
The resulting transport geodesic is defined in Wasserstein distribution space, 
rather than on a prescribed input-space data manifold.
Together, these components define our transport-geodesic attribution framework.

% This story differs from an ``on-manifold Shapley'' story. We do not claim that every intermediate point of every sample trajectory lies on a true low-dimensional image manifold. Such a claim would require a formal data manifold, a certificate of membership, and a proof that the learned sampler respects it. Our construction does not need this claim. We only require a reference distribution, a data distribution, and a transport flow between them. The intermediate marginals are generated by this flow. For this reason, we use terms such as \emph{transport-consistent}, \emph{generative}, and \emph{transport-geodesic} rather than strict on-manifold.

The practical method follows the same principle. The ideal Wasserstein geodesic is not available in high-dimensional image spaces.
To instantiate this framework, we use Rectified Flow
to learn a time-dependent velocity field transporting
samples from a reference distribution toward the
data distribution \citep{liu2022flow,lipman2022flow}.
Given an input \(\vx\), 
we integrate the learned flow
backward to recover an input-conditioned reference
endpoint, then evaluate the predictor's gradients
along the corresponding characteristic trajectory.
Accumulating coordinatewise gradient contributions
along this path provides a practical approximation
to transport-geodesic Aumann--Shapley attribution.

\paragraph{We make four contributions.} 

\begin{itemize}
    \item First, we make explicit a separation that is implicit in path-based attribution methods: fixed-path credit allocation is a different problem from path selection. 
    \item Second, we prove a fixed-path uniqueness theorem showing that the Aumann-Shapley line integral is the unique attribution rule under explicit fixed-path axioms and coordinate-trace separability. 
    \item Then, we formulate distribution-level attribution path selection through kinetic-action minimization in dynamic optimal transport.
    % Third, we explore a distribution-level path-selection principle based on kinetic-action minimization, which turns baseline choice into an optimal transport problem. 
    \item Finally, we instantiate the framework with Rectified Flow, derive attribution stability bounds, and demonstrate improved cross-seed stability and spatial structure with competitive deletion faithfulness.
\end{itemize}

% Due to the page limit, we present a detailed related work section in the appendix. 

% \input{section/02_related_work}
\section{Problem setup}
\label{sec:problem_setup}

Let \(f_c:\R^d\to\R\) be the scalar score of a fixed predictor for class or target \(c\). We write inputs as bold vectors \(\vx\in\R^d\). We write \(\partial_i f(\vx)\) for the partial derivative of a scalar function with respect to coordinate \(x_i\), and we write \(\nabla_{\vx} f(\vx)\) for the full input gradient. When the input variable is clear, we also write \(\nabla f(\vx)\). Let \(p_1\in\cP_2(\R^d)\) denote the data distribution and let \(p_0\in\cP_2(\R^d)\) denote a reference distribution such as a standard Gaussian or another simple prior. Here \(\cP_2(\R^d)\) is the set of probability measures with finite second moment. For a specific observed input \(\vx_1\), an explanation should account for the score change between a reference endpoint \(\vx_0\) and \(\vx_1\). The endpoint \(\vx_0\) can be produced by a backward generative flow or by a coupling between \(p_0\) and \(p_1\).

We describe a counterfactual transition by a continuously differentiable path
\[
\gamma:[0,1]\to\R^d,\qquad \gamma(0)=\vx_0,\quad \gamma(1)=\vx_1.
\]
The parameter \(t\in[0,1]\) is the path time. For a \(C^1\) path, we write
$\dot{\gamma}(t)=\frac{\mathrm{d}\gamma(t)}{\mathrm{d}t}$
for its time derivative, and \(\dot{\gamma}_i(t)=\mathrm{d}\gamma_i(t)/\mathrm{d}t\) for the derivative of its \(i\)-th coordinate. A \(C^1\) path is called regular when \(\dot{\gamma}(t)\ne\bm{0}\) for all \(t\in[0,1]\). This dot notation is only shorthand; in the main attribution formula we write the derivative explicitly. The path may be a straight line, a diffusion trajectory, a flow trajectory, or an optimal transport characteristic.

A path-based attribution rule returns a vector \(\vphi(f_c,\gamma)\in\R^d\) whose \(i\)-th coordinate represents the contribution of input coordinate \(i\) to the score change along \(\gamma\). Efficiency asks that these contributions sum to the finite difference:
\begin{equation}
\sum_{i=1}^d \phi_i(f_c,\gamma)=f_c(\gamma(1))-f_c(\gamma(0)).
\label{eq:efficiency}
\end{equation}
Equation~\eqref{eq:efficiency} is the continuous analogue of Shapley efficiency. It is necessary but not sufficient. It says that the accounting balances, but it does not say how the accounting should be done or which path should be used.

The central design principle of our method is not to design a new saliency
heuristic, but to separate two decisions that are often entangled. For an input
\(\vx\), write \(\gamma_{\vx}\) for a counterfactual path ending at \(\vx\),
write \(\Phi(f,\gamma)\) for the Aumann-Shapley attribution vector along a fixed
path, and write \(\operatorname{Char}_{\vx}(\vv)\) for the characteristic curve of
a velocity field \(\vv\) that ends at \(\vx\). Then the organizing equation is
\begin{tightdisplay}
\[
\operatorname{Attr}(f,\vx)
=
\underbrace{
\creditbox{\Phi(f,\gamma_{\vx})}
}_{\scriptsize
\begin{array}{c}
\text{Aumann-Shapley}\\
\text{credit allocation}
\end{array}}
\quad
\text{with}
\quad
\underbrace{
\pathbox{
\begin{gathered}
(\rho^\star,\vv^\star)
\in
\arg\min_{(\rho,\vv):p_0\to p_1}\cA(\rho,\vv),\\[-0.5mm]
\gamma_{\vx}=\operatorname{Char}_{\vx}(\vv^\star)
\end{gathered}}
}_{\scriptsize
\begin{array}{c}
\text{transport-geodesic}\\
\text{path selection}
\end{array}} .
\]
\end{tightdisplay}%
Here \(\cA\) denotes kinetic action, which we define formally in Eq.~\eqref{eq:kinetic_action}. The blue term asks how to allocate score change once a path is fixed. This is the part addressed by Aumann-Shapley path integrals. The orange term asks which counterfactual path should be used. This is the part that is often fixed by a baseline or a hand-designed interpolation rule; we instead pose it as an optimal generative transport problem. 

A time-dependent velocity field \(\vv_t:\R^d\to\R^d\) induces trajectories through the ordinary differential equation
\begin{equation}
\frac{\mathrm{d}\vx_t}{\mathrm{d}t}=\vv_t(\vx_t),\qquad t\in[0,1].
\label{eq:ode}
\end{equation}

For any time-dependent state \(\vx_t\), the shorthand \(\dot{\vx}_t\) means \(\mathrm{d}\vx_t/\mathrm{d}t\). For a flow transporting \(p_0\) to \(p_1\), let
\(\rho_t\) denote the marginal law of the solution of Eq.~\eqref{eq:ode} initialized with \(\vx_0\sim p_0\). Mass conservation and the prescribed endpoint marginals are expressed as 
\begin{equation}
\partial_t\rho_t+\nabla\cdot(\rho_t\vv_t)=0,
\qquad \rho_0=p_0,
\qquad \rho_1=p_1.
\label{eq:continuity}
\end{equation}

Equation~\eqref{eq:continuity} is a distribution-level constraint. It does not say that \(\rho_t=p_1\) for intermediate times. The intermediate laws \(\rho_t\) form a generative bridge between the reference and data distributions.

We will use two types of objects. A fixed path \(\gamma\) is the object used by the attribution integral. A flow pair \((\rho,\vv)\), where \(\rho=(\rho_t)_{t\in[0,1]}\) and \(\vv=(\vv_t)_{t\in[0,1]}\), is the object used to select paths. In experiments, we parameterize the learned velocity field by a neural network \(\vv_{\vtheta}(\vx,t)\). When we compare theory and implementation, we write \(\hat{\vv}_t(\vx)=\vv_{\vtheta}(\vx,t)\) for the learned field. Its characteristic curves supply the numerical paths for attribution.

\section{Axiomatic attribution along a fixed path}
\label{sec:axiomatic_path}

We first solve the allocation problem while treating the path as fixed. This part of the theory is deliberately independent of optimal transport. It says that, after the user or a generative model has chosen a smooth path, the natural attribution rule is forced by axioms.

\begin{definition}[Admissible score class and path attribution rule]
\label{def:path_rule}
Let \(\cF\) be an admissible class of scalar scores. In this paper, admissible means that every \(f\in\cF\) is \(C^1\) on an open set containing the image of the path under consideration, and that \(\cF\) is closed under finite linear combinations. For a \(C^1\) path \(\gamma:[0,1]\to\R^d\), a path attribution rule assigns a number \(A_i(f,\gamma)\in\R\) to each score \(f\in\cF\) and each coordinate \(i\in\{1,\ldots,d\}\). The vector \(\bm{A}(f,\gamma)=(A_1(f,\gamma),\ldots,A_d(f,\gamma))\) explains the score difference \(f(\gamma(1))-f(\gamma(0))\) along \(\gamma\).
\end{definition}

For a score \(f\in\cF\), a path \(\gamma\), and a coordinate \(i\), define the coordinate trace
\[
h_i^f(t)=\frac{\partial f(\gamma(t))}{\partial x_i}\frac{\mathrm{d}\gamma_i(t)}{\mathrm{d}t},\qquad t\in[0,1].
\]
This trace is the infinitesimal score change attributed to coordinate \(i\) at path time \(t\). The following assumption collects the fixed-path axioms. We state the axioms as an assumption because the representation theorem needs to know exactly what information the rule is allowed to use.

\begin{assumption}[Fixed-path attribution axioms]
\label{ass:fixed_path_axioms}
For every fixed regular \(C^1\) path \(\gamma\), the rule \(A_i(f,\gamma)\) satisfies the following conditions for all admissible scores. It satisfies efficiency, namely Eq.~\eqref{eq:efficiency}. It is linear in the model score: for real numbers \(a,b\) and scores \(f,g\in\cF\),
\[
A_i(af+bg,\gamma)=aA_i(f,\gamma)+bA_i(g,\gamma).
\]
It satisfies the dummy property: if \(h_i^f(t)=0\) for all \(t\in[0,1]\), then \(A_i(f,\gamma)=0\). It is invariant to smooth increasing reparameterizations of the path: for any \(C^1\) bijection \(\sigma:[0,1]\to[0,1]\) with \(\sigma(0)=0\), \(\sigma(1)=1\), and \(\mathrm{d}\sigma(t)/\mathrm{d}t>0\), we have \(A_i(f,\gamma\circ\sigma)=A_i(f,\gamma)\). Finally, it is coordinate-trace determined and continuous: for each \(i\), the value \(A_i(f,\gamma)\) depends on \(f\) along \(\gamma\) only through the scalar trace \(h_i^f\), and this dependence is continuous under uniform convergence of that trace.
\end{assumption}

The coordinate-trace condition is not a technical trick. It encodes what a coordinate attribution along a path means. If the attribution assigned to coordinate \(i\) changed when all coordinate-\(i\) infinitesimal contributions along the path stayed the same, then the rule would be using information outside the coordinate trace to assign coordinate-\(i\) credit. That behavior would no longer be a coordinatewise path attribution.

\begin{assumption}[Coordinate-trace separability]
\label{ass:trace_richness}
For a fixed path $\gamma$ and coordinate $i$, let
\[
H_i(\gamma)
=
\{h_i^f : f \in \mathcal{F}\}
\]
denote the set of admissible coordinate traces.
For every $h \in H_i(\gamma)$, there exists
$g \in \mathcal{F}$ such that
\[
h_i^g = h, \qquad
h_j^g = 0 \quad \text{for all } j \neq i.
\]
\end{assumption}

This condition requires separation only of
realizable coordinate traces, rather than
arbitrarily prescribed traces.
For self-intersecting paths, repeated spatial
points impose consistency constraints on the
gradients of a smooth score function, so
separability need not hold.
Nevertheless, the Aumann--Shapley integral
remains well defined even when separability fails.
For smooth embedded paths and sufficiently rich
score classes, separability follows from suitable
extensions of first-order information along
the curve. This includes the nondegenerate straight
displacement paths arising under ideal quadratic
optimal transport.

% \begin{assumption}[Coordinate-trace richness]
% \label{ass:trace_richness}
% For the fixed path \(\gamma\), the admissible score class \(\cF\) is rich enough to separate coordinate traces. Concretely, for any coordinate \(i\) and any realizable coordinate trace \(h=h_i^f\) with \(f\in\cF\), there is an admissible score \(g\in\cF\) such that
% \[
% \frac{\partial g(\gamma(t))}{\partial x_i}\frac{\mathrm{d}\gamma_i(t)}{\mathrm{d}t}=h(t)
% \quad\text{and}\quad
% \frac{\partial g(\gamma(t))}{\partial x_j}\frac{\mathrm{d}\gamma_j(t)}{\mathrm{d}t}=0\quad\text{for all }j\ne i.
% \]
% \end{assumption}

% Assumption~\ref{ass:trace_richness} makes explicit a richness requirement on realizable coordinate traces, rather than arbitrarily prescribed traces. It requires that each such trace can be retained while all other coordinate traces are set to zero. For smooth embedded paths and sufficiently rich score classes, this condition can be obtained by extending compatible first-order information along the curve. The role of self-intersections is discussed in Appendix~\ref{app:transport_details}.

\begin{definition}[Aumann-Shapley path attribution]
\label{def:as_path}
Let \(f\in\cF\) be differentiable in a neighborhood of the image of \(\gamma\). The Aumann-Shapley attribution of coordinate \(i\) along \(\gamma\) is
\begin{equation}
\Phi_i(f,\gamma)=
\int_0^1
\frac{\partial f(\gamma(t))}{\partial x_i}
\frac{\mathrm{d}\gamma_i(t)}{\mathrm{d}t}
\dd t.
\label{eq:as_path}
\end{equation}
\end{definition}

The definition is the coordinate decomposition of the line integral of \(\nabla f\) along \(\gamma\). Summing Eq.~\eqref{eq:as_path} over coordinates gives
\[
\sum_{i=1}^d\Phi_i(f,\gamma)=
\int_0^1 \nabla f(\gamma(t))^\top\frac{\mathrm{d}\gamma(t)}{\mathrm{d}t}\dd t
=f(\gamma(1))-f(\gamma(0)),
\]
so efficiency follows from the chain rule. Reparameterization invariance follows because a change of the time variable changes the path derivative and the integration variable in compensating ways.

\begin{theorem}[Fixed-path uniqueness]
\label{thm:fixed_path_uniqueness}
Fix a regular \(C^1\) path \(\gamma\). Any path attribution rule satisfying Assumptions~\ref{ass:fixed_path_axioms} and~\ref{ass:trace_richness} coincides with the Aumann-Shapley path attribution in Eq.~\eqref{eq:as_path}. That is, for every \(f\in\cF\) and every coordinate \(i\),
\[
A_i(f,\gamma)=\Phi_i(f,\gamma).
\]
\end{theorem}

% Theorem~\ref{thm:fixed_path_uniqueness} removes one degree of freedom from the explanation problem. Once a path is fixed, the credit-allocation rule is not an additional design choice. The remaining question is therefore not how to modify the attribution formula, but which path should be used.

Theorem~\ref{thm:fixed_path_uniqueness} removes one degree
of freedom from the explanation problem. Under its assumptions, fixing the path also fixes the credit-allocation rule. The remaining question is therefore not how to modify the attribution formula, but which path should be used.

This result also clarifies the relation to Integrated Gradients. Integrated Gradients is Eq.~\eqref{eq:as_path} for the straight path \(\gamma(t)=\vx_0+t(\vx_1-\vx_0)\). Our method keeps the same fixed-path allocation principle, but replaces a hand-specified reference-and-path rule with a distribution-level transport coupling and its induced characteristic.

\section{Transport-geodesic path selection by optimal generative flows}
\label{sec:geodesic_flows}

We now address the path-selection problem. 
Given the reference distribution $p_0$ and data distribution $p_1$, 
many possible couplings and transport paths remain. 
We adopt kinetic-action minimization as a predictor-independent
variational principle for resolving this ambiguity.
Specifically, among all admissible flows transporting $p_0$ to $p_1$,
we select the one minimizing the time-integrated expected squared
velocity, independently of the predictor's gradients used for
credit allocation.
By the dynamic formulation of quadratic optimal transport,
this yields a Wasserstein-2 geodesic and an associated optimal
endpoint coupling, whose characteristics define the attribution paths.
\footnote{We further discuss the scope and limitations
of this path-selection principle in
Section~\ref{sec:discussion}.}

\begin{assumption}[Transport regularity]
\label{ass:transport_regular}
The reference and data distributions \(p_0,p_1\in\cP_2(\R^d)\) are absolutely continuous with respect to Lebesgue measure on their supports, and the quadratic-cost optimal transport problem between them admits a unique optimal dynamic plan. The corresponding velocity field \(\vv_t^\star\) is locally Lipschitz on the compact region that contains the trajectories used for attribution.
\end{assumption}

Assumption~\ref{ass:transport_regular} is a standard regularity condition for the ideal theory. It is stronger than what we can certify for high-dimensional learned image models. We use it to define the target object. The learned Rectified Flow is an approximation to this target, not a proof that exact optimal transport has been recovered.

We write \(W_2(p_0,p_1)\) for the quadratic Wasserstein distance between \(p_0\) and \(p_1\). For any admissible flow pair \((\rho,\vv)\), with \(\rho=(\rho_t)_{t\in[0,1]}\) and \(\vv=(\vv_t)_{t\in[0,1]}\), satisfying Eq.~\eqref{eq:continuity}, define its kinetic action by
\begin{equation}
\cA(\rho,\vv)=\int_0^1\int_{\R^d}\|\vv_t(\vx)\|_2^2\dd\rho_t(\vx)\dd t.
\label{eq:kinetic_action}
\end{equation}
When \(\rho_t\) admits a density, this measure integral is the same as \(\int \|\vv_t(\vx)\|_2^2\rho_t(\vx)\dd\vx\). The Benamou-Brenier formula states that
\begin{equation}
W_2^2(p_0,p_1)=\inf_{(\rho,\vv)\text{ satisfying Eq.~\eqref{eq:continuity}}}\cA(\rho,\vv).
\label{eq:benamou_brenier}
\end{equation}
The minimizer \((\rho_t^\star,\vv_t^\star)\) is the Wasserstein-2 geodesic in distribution space \citep{benamou2000computational,villani2008optimal}. This is the sense in which we use the word transport-geodesic. The path of distributions is shortest in the Wasserstein geometry. A single sample trajectory is a characteristic curve of this distributional flow; under ideal quadratic transport, it is the straight displacement interpolation associated with the optimal coupling. In the ideal quadratic OT setting, the resulting attribution therefore coincides with Integrated Gradients from an optimally coupled, input-conditioned reference endpoint.

\begin{definition}[Transport-geodesic characteristic path]
\label{def:geodesic_path}
Under Assumption~\ref{ass:transport_regular}, let \((\rho^\star,\vv^\star)\) minimize Eq.~\eqref{eq:kinetic_action}. For an observed endpoint \(\vx_1\) in the support of \(p_1\), the geodesic characteristic path \(\gamma_{\vx_1}^\star\) is the solution of
\begin{equation}
\frac{\mathrm{d}\gamma_{\vx_1}^\star(t)}{\mathrm{d}t}=\vv_t^\star(\gamma_{\vx_1}^\star(t)),\qquad \gamma_{\vx_1}^\star(1)=\vx_1,
\label{eq:geodesic_ode}
\end{equation}
traced backward to a reference endpoint \(\gamma_{\vx_1}^\star(0)=\vx_0\).
\end{definition}

If \(\vx_0\sim p_0\) and \(\vx_t\) follows Eq.~\eqref{eq:geodesic_ode} forward in time, then \(\vx_t\sim\rho_t^\star\) for each \(t\). Notice the precise statement: \(\rho_t^\star\) interpolates between \(p_0\) and \(p_1\). We do not need or claim \(\rho_t^\star=p_1\) for all \(t\).

\begin{definition}[Transport-geodesic Aumann-Shapley attribution]
\label{def:geodesic_as}
For a differentiable target score \(f_c\) and an input \(\vx_1\), the transport-geodesic Aumann-Shapley attribution is
\begin{equation}
\Psi_i(f_c,\vx_1)=
\int_0^1
\frac{\partial f_c(\gamma_{\vx_1}^\star(t))}{\partial x_i}
\frac{\mathrm{d}\gamma_{\vx_1,i}^\star(t)}{\mathrm{d}t}
\dd t.
\label{eq:geodesic_as}
\end{equation}
\end{definition}

Definition~\ref{def:geodesic_as} combines the two parts. The Aumann-Shapley integral allocates credit along a path, and the optimal transport problem selects the path. The explanation is canonical only relative to the ideal transport problem and its regularity assumptions. It is best understood as a principled transport alternative to baseline paths and Riemannian path choices, not as a universal dominance claim.
This definition separates two roles that are often entangled. The generative
transport flow determines where the explanation path goes, while the predictor
\(f_c\) determines how the score changes along that path. Thus, the path is not
chosen by the same sensitivity geometry that is being explained.

\begin{theorem}[Axiomatic and transport-geodesic characterization]
\label{thm:geodesic_characterization}
Suppose Assumptions~\ref{ass:fixed_path_axioms},
\ref{ass:trace_richness}, and~\ref{ass:transport_regular}
hold, and the selected characteristic path is regular. Among all attribution rules that first choose a kinetic-action-minimizing flow from \(p_0\) to \(p_1\) and then apply a fixed-path attribution rule to its characteristic path, Eq.~\eqref{eq:geodesic_as} is the unique rule satisfying the fixed-path axioms.
\end{theorem}

% Theorem~\ref{thm:geodesic_characterization} says that there are two sources of uniqueness. The fixed-path axioms force the line integral, while the kinetic-action principle chooses the distributional path. This is the precise version of the paper's main claim. It is not a statement of strict manifold membership, and it does not impose a total ordering over all attribution methods.

Under the stated assumptions, Theorem~\ref{thm:geodesic_characterization} combines two sources of uniqueness: the fixed-path axioms and coordinate-trace separability determine the line integral, while the kinetic-action principle selects the unique distributional path.

The construction reduces to familiar cases. If \(\gamma\) is the straight line from a fixed baseline to \(\vx_1\), Eq.~\eqref{eq:as_path} recovers Integrated Gradients. If the model is additive, namely \(f(\vx)=b+\sum_i f_i(x_i)\), Eq.~\eqref{eq:geodesic_as} returns the coordinatewise finite differences \(f_i(x_{1,i})-f_i(x_{0,i})\), which match the classical Shapley allocation for the induced additive game.

We also need a stability guarantee because the ideal flow is not observed. Let \(\hat{\vv}_t(\vx)=\vv_{\vtheta}(\vx,t)\) denote the learned vector field. For an observed endpoint \(\vx_1\), let \(\hat{\gamma}_{\vx_1}\) be the learned characteristic defined by
\[
\frac{\mathrm{d}\hat{\gamma}_{\vx_1}(t)}{\mathrm{d}t}=\hat{\vv}_t(\hat{\gamma}_{\vx_1}(t)),\qquad \hat{\gamma}_{\vx_1}(1)=\vx_1.
\]
Define \(\hat{\Psi}_i\) by replacing \(\gamma_{\vx_1}^\star\) with \(\hat{\gamma}_{\vx_1}\) in Eq.~\eqref{eq:geodesic_as}.

\begin{assumption}[Stability regularity]
\label{ass:stability_regular}
The ideal and learned trajectories remain in a compact set \(\cK\subset\R^d\). On \(\cK\), the target score \(f_c\) has bounded gradient and Lipschitz gradient, and both vector fields are uniformly bounded and uniformly Lipschitz in \(\vx\) on \([0,1]\times\cK\). Moreover,
\[
\sup_{t\in[0,1],\vx\in\cK}\|\hat{\vv}_t(\vx)-\vv_t^\star(\vx)\|_2\le\eps.
\]
\end{assumption}

\begin{theorem}[Stability under flow approximation]
\label{thm:stability}
Under Assumption~\ref{ass:stability_regular}, there is a constant \(C\), depending only on the compact set, the time horizon, the uniform bounds and spatial Lipschitz constants of the vector fields, and the gradient bound and gradient Lipschitz constant of \(f_c\), such that for every coordinate \(i\),
\begin{equation}
|\Psi_i(f_c,\vx_1)-\hat{\Psi}_i(f_c,\vx_1)|\le C\eps.
\label{eq:stability_bound}
\end{equation}
\end{theorem}

Theorem~\ref{thm:stability} gives the main engineering message. Improving the learned vector field improves the attribution in a controlled way. The theorem does not say that the learned path is exact optimal transport. It says that, when the learned vector field approaches the ideal field on the relevant region, the attribution approaches the ideal transport-geodesic attribution.

\section{Implementation and experiments}
\label{sec:experiments}

We instantiate the ideal construction with Rectified Flow and Reflow. Rectified Flow learns a time-dependent vector field \(\vv_{\vtheta}(\vx,t)\) that moves samples from a reference distribution to the data distribution. Given paired samples \((\vz_0,\vz_1)\sim\pi\), where \(\pi\) is a coupling between \(p_0\) and \(p_1\), the basic training objective is
\begin{equation}
\min_{\vtheta}\;\E_{t,\vz_0,\vz_1}\left\|\vv_{\vtheta}((1-t)\vz_0+t\vz_1,t)-(\vz_1-\vz_0)\right\|_2^2.
\label{eq:rectified_flow_loss}
\end{equation}
Reflow refines the coupling by first generating trajectories with a learned flow and then retraining on the induced endpoint pairs \citep{liu2022flow}. We use Rectified Flow and Reflow to construct tractable transport paths and assess their approximation to the quadratic OT target through kinetic-action and oracle-based diagnostics. Standard Reflow does not in general guarantee recovery of the quadratic optimal transport coupling. This procedure is useful for us because lower-curvature and lower-action paths are closer to the geodesic ideal in Eq.~\eqref{eq:benamou_brenier}.

For a target input \(\vx_1\), we integrate the learned ODE backward to obtain a reference endpoint \(\hat{\vx}_0\), then integrate forward to obtain states \(\hat{\vx}_0,\ldots,\hat{\vx}_K\). Before accumulating attributions, we set the final stored state to the observed input, \(\hat{\vx}_K\leftarrow\vx_1\).Throughout the experiments, \(K\) denotes the number of attribution intervals, \(t_k=k/K\), \(\Delta t=1/K\), and \(\hat{\vx}_k\approx\hat{\gamma}_{\vx_1}(t_k)\). We compute gradients of the target score at these states and approximate Eq.~\eqref{eq:geodesic_as} by a Riemann sum:
\begin{equation}
\hat{\Psi}_i(f_c,\vx_1)=\sum_{k=0}^{K-1}
\frac{\partial f_c(\hat{\vx}_k)}{\partial x_i}
(\hat{x}_{k+1,i}-\hat{x}_{k,i}).
\label{eq:discrete_estimator}
\end{equation}
We use the same number of attribution intervals for path-based baselines whenever possible. The discrete estimator requires \(K\) predictor gradient evaluations, while generative path construction adds solver-dependent ODE evaluations. This is much cheaper than exact discrete Shapley, whose cost grows exponentially in the number of input features. In Alg.~\ref{alg:geodesic_as}, \(\odot\) denotes coordinatewise multiplication.

\begin{algorithm}[t]
\caption{Transport-Geodesic Aumann-Shapley attribution with a learned flow}
\label{alg:geodesic_as}
\begin{algorithmic}[1]
\Require target score \(f_c\), input \(\vx_1\), learned vector field \(\vv_{\vtheta}\), number of steps \(K\)
\State integrate \(\mathrm{d}\vx_t/\mathrm{d}t=\vv_{\vtheta}(\vx_t,t)\) backward from \(\vx_1\) to obtain \(\hat{\vx}_0\)
\State integrate the same ODE forward from \(\hat{\vx}_0\) and store \(\hat{\vx}_0,\ldots,\hat{\vx}_K\)
\State set \(\hat{\vx}_K\leftarrow\vx_1\)
\State initialize \(\hat{\bm{\Psi}}=\bm{0}\)
\For{\(k=0,\ldots,K-1\)}
    \State compute \(\vg_k=\nabla_{\vx} f_c(\hat{\vx}_k)\)
    \State update \(\hat{\bm{\Psi}}\leftarrow\hat{\bm{\Psi}}+\vg_k\odot(\hat{\vx}_{k+1}-\hat{\vx}_k)\)
\EndFor
\State \Return \(\hat{\bm{\Psi}}\)
\end{algorithmic}
\end{algorithm}

\subsection{Evaluation questions}
\label{sec:evaluation_questions}

We organize experiments around four questions. These questions are meant to test the transport path-selection hypothesis, not to establish a universal ranking of attribution methods. 
\begin{enumerate*}[label=\arabic*)]
\item Does the numerical integral satisfy efficiency as the step count increases? 
\item Does Reflow reduce the geometry gap relative to less rectified flows? 
\item Does attribution error track vector-field error as predicted by Theorem~\ref{thm:stability}? 
\item On real images, do transport-consistent paths give structured explanations without destroying deletion faithfulness?
\end{enumerate*}

\subsection{Experimental setup}
\label{sec:experimental_setup}

We use CUB-200-2011 for the numerical completeness study
\citep{WahCUB_200_2011}, controlled flow checkpoints for the
stability analysis, and CIFAR-10 and CelebA-HQ for the image
benchmarks
\citep{Krizhevsky09learningmultiple,DBLP:journals/corr/abs-1710-10196}.

We compare with SmoothGrad \citep{DBLP:journals/corr/SmilkovTKVW17}, Guided Backpropagation \citep{springenberg2014striving}, GradientSHAP baselines \citep{lundberg2017unified}, Integrated Gradients \citep{sundararajan2017axiomatic}, and DDIM generative paths \citep{song2020denoising}. We report standard deletion metrics together with path diagnostics.

We use $32\times32$ inputs for CUB-200-2011 and CIFAR-10 and
$256\times256$ inputs for CelebA-HQ, with ResNet-18 predictors
throughout. Additional implementation details are provided in
Appendix~\ref{app:rf_reflow_details}.

\subsection{Numerical completeness}
\label{sec:numerical_completeness}

Efficiency requires the attribution sum to match the score change. For \(N\) evaluated examples, we measure the residual
\[
R_{\mathrm{eff}}=\frac{1}{N}\sum_{j=1}^N\left|\sum_{i=1}^d\hat{\Psi}_i^{(j)}-\left(f_c(\vx_1^{(j)})-f_c(\hat{\vx}_0^{(j)})\right)\right|.
\]
Table~\ref{tab:completeness} shows convergence as the number of attribution intervals increases. The default \(K=50\) gives a practical balance, while \(K=100\) or \(K=200\) can be used when tighter efficiency residuals are desired.

\begin{table}[t]
\centering
\caption{Completeness residual for the discrete estimator in Eq.~\eqref{eq:discrete_estimator}. The residual decreases as the integration grid becomes finer.}
\label{tab:completeness}
\small
\begin{tabular}{ccccc}
\toprule
Steps \(K\) & MAE \(\downarrow\) & Std. dev. & SEM & Relative error \(\downarrow\) \\
\midrule
10  & 1.483 & 1.346 & 0.177 & 19.30\% \\
20  & 0.895 & 0.778 & 0.102 & 11.65\% \\
50  & 0.411 & 0.317 & 0.042 & 5.34\% \\
100 & 0.229 & 0.181 & 0.024 & 2.98\% \\
200 & 0.101 & 0.079 & 0.010 & 1.34\% \\
\bottomrule
\end{tabular}
\end{table}

\subsection{Path geometry and stability}
\label{sec:path_geometry_stability}

We next ask whether a more geodesic transport path produces a more stable attribution. We compare a one-round Rectified Flow model (1-RF)
with its reflowed counterpart (2-RF) over three seeds. We estimate the mean discrete kinetic action over \(N\) evaluated trajectories for each seed by
\[
\widehat{\cA}=
\frac{1}{N}\sum_{j=1}^{N}\sum_{k=0}^{K-1}
\frac{\|\hat{\vx}_{k+1}^{(j)}-\hat{\vx}_k^{(j)}\|_2^2}{\Delta t}.
\]
We measure attribution stability with pixelwise variance, SSIM between attribution maps, and rank correlation of feature scores. Table~\ref{tab:stability_metrics} shows that Reflow moderately lowers action and substantially improves stability. We interpret this as evidence that path geometry matters for explanation stability. We avoid saying that this proves strict manifold adherence.

\begin{table}[t]
\centering
\caption{Path action and attribution stability across seeds. Lower-action reflowed paths (2-RF) produce more stable maps and more consistent feature rankings.}
\label{tab:stability_metrics}
\small
\begin{tabular}{lcccc}
\toprule
Method & Action \(\widehat{\cA}\) \(\downarrow\) & Pixel var. \(\downarrow\) & SSIM \(\uparrow\) & Rank corr. \(\uparrow\) \\
\midrule
1-RF & 3179.2 \(\pm\) 295.8 & 0.0032 \(\pm\) 0.0021 & 0.716 \(\pm\) 0.080 & 0.662 \(\pm\) 0.087 \\
2-RF & \textbf{3006.9} \(\pm\) \textbf{340.4} & \textbf{0.0010} \(\pm\) \textbf{0.0008} & \textbf{0.911} \(\pm\) \textbf{0.051} & \textbf{0.882} \(\pm\) \textbf{0.061} \\
\bottomrule
\end{tabular}
\end{table}

% The stability bound in Theorem~\ref{thm:stability} predicts that attribution error should decrease when the learned vector field approaches a stronger oracle field. We treat a converged flow checkpoint as an oracle and compare earlier checkpoints against it.
Motivated by Theorem~\ref{thm:stability}, we compare earlier flow checkpoints with a converged reference checkpoint. This measures agreement with the reference model rather than error against the unknown ideal OT field. Let \(\bm{\Psi}^{\mathrm{ref}}\) and \(\vv_{t_k}^{\mathrm{ref}}\) denote the attribution and velocity field of this reference checkpoint. We use \(\|\hat{\bm{\Psi}}-\bm{\Psi}^{\mathrm{ref}}\|_2/(\|\bm{\Psi}^{\mathrm{ref}}\|_2+10^{-12})\) as the relative attribution error and
\[
E_{\mathrm{field}}=
\frac{\left(\sum_{k=0}^{K-1}\|\hat{\vv}_{t_k}(\hat{\vx}_k)-\vv_{t_k}^{\mathrm{ref}}(\hat{\vx}_k)\|_2^2\right)^{1/2}}{\left(\sum_{k=0}^{K-1}\|\vv_{t_k}^{\mathrm{ref}}(\hat{\vx}_k)\|_2^2\right)^{1/2}+10^{-12}}
\]
as the empirical field mismatch relative to this reference. Across checkpoints, the median per-sample Pearson correlation between field mismatch and attribution discrepancy exceeds \(0.95\).
This supports an empirical association between the two reference-relative quantities. Full scatter plots are reported in the appendix.

% as the empirical flow approximation error. Across checkpoints, these two quantities show an approximately linear relation, with median Pearson correlation above \(0.95\). This supports the narrower claim that generative flow quality controls the reliability of the transport-based attribution we study. We keep the full scatter plots for the appendix so that the main text can focus on the central evidence in Table~\ref{tab:stability_metrics}. 

\subsection{Image benchmarks}
\label{sec:image_benchmarks}

For CIFAR-10 and CelebA-HQ, we report both path diagnostics and standard explanation metrics. The geometric path straightness score is
\[
\mathrm{GPS}=\frac{\sum_{k=0}^{K-1}\|\hat{\vx}_{k+1}-\hat{\vx}_k\|_2}{\|\hat{\vx}_K-\hat{\vx}_0\|_2}.
\]
A value near one indicates a nearly straight discrete path in ambient space. The Flow Consistency Error is
\begin{equation}
\mathrm{FCE}=\frac{1}{K}\sum_{k=0}^{K-1}\left\|\frac{\hat{\vx}_{k+1}-\hat{\vx}_k}{\Delta t}-\hat{\vv}_{t_k}(\hat{\vx}_k)\right\|_2^2.
\label{eq:fce}
\end{equation}
Equation~\eqref{eq:fce} measures consistency with the learned
transport field, not attribution faithfulness or membership
in the true data distribution. Since Transport Flow uses
this same field to generate its paths, FCE is an internal
path-dynamics diagnostic and is not used to rank attribution
methods. For spatial structure, we use a Structure-Aware Total Variation diagnostic,
\begin{equation}
\mathrm{SATV}(\vphi)=
\sum_{(u,v)\in\Omega}
\|\nabla_{\mathrm{img}}\vphi_{u,v}\|_1
\exp\left(-\alpha\|\nabla_{\mathrm{img}} I_{u,v}\|_2\right),
\label{eq:satv}
\end{equation}
where \(\Omega\) is the pixel grid, \(I\) is the input image, and \(\vphi\) is the channel-aggregated and normalized attribution map defined in Appendix~\ref{app:satv_details}. Here, \(\nabla_{\mathrm{img}}\) denotes finite differences on the image grid, and \(\alpha=10\). Equation~\eqref{eq:satv} penalizes high-frequency saliency variation in flat image regions while allowing changes near image edges. We also report the Edge Alignment Score (EAS), which measures whether attribution-map edges align with image edges, and deletion scores under zero and blur replacement.

\begin{table}[t]
\centering
\caption{CIFAR-10 and CelebA-HQ benchmark results.
We report path diagnostics only for methods that define a
continuous path. FCE measures agreement with our learned
transport field and is not used to rank attribution methods;
deletion scores provide perturbation-based faithfulness
evidence.}
\label{tab:main_benchmark}
\scriptsize
\setlength{\tabcolsep}{2.4pt}
\resizebox{\linewidth}{!}{%
\begin{tabular}{c l c c c c c c}
\toprule
Data & Method & GPS & FCE & SATV \(\downarrow\) & EAS \(\uparrow\) & Del. zero \(\downarrow\) & Del. blur \(\downarrow\) \\
\midrule
\multirow{6}{*}{CIFAR-10}
& SmoothGrad & - & - & 0.022 \(\pm\) 0.008 & 0.363 \(\pm\) 0.117 & 0.525 \(\pm\) 0.923 & 1.216 \(\pm\) 1.202 \\
& GuidedBackprop & - & - & 0.040 \(\pm\) 0.018 & 0.439 \(\pm\) 0.128 & 0.966 \(\pm\) 0.978 & 1.627 \(\pm\) 1.064 \\
& GradientSHAP & - & - & 0.669 \(\pm\) 0.408 & 0.116 \(\pm\) 0.146 & 0.691 \(\pm\) 0.975 & 1.978 \(\pm\) 1.050 \\
\cmidrule(lr){2-8}
& Integrated Gradients & 1.000 \(\pm\) 0.000 & \((1.0\pm0.1)\times10^4\) & 0.643 \(\pm\) 0.404 & 0.114 \(\pm\) 0.134 & 0.670 \(\pm\) 0.965 & 1.931 \(\pm\) 1.105 \\
& DDIM & 1.059 \(\pm\) 0.017 & \((3.1\pm0.5)\times10^3\) & 0.073 \(\pm\) 0.062 & 0.119 \(\pm\) 0.160 & \textbf{0.444} \(\pm\) \textbf{0.940} & 1.360 \(\pm\) 1.143 \\
\rowcolor{oursred}\cellcolor{white} & Transport Flow & 1.023 \(\pm\) 0.008 & 0.005 \(\pm\) 0.003 & \textbf{0.062} \(\pm\) \textbf{0.048} & \textbf{0.119} \(\pm\) \textbf{0.138} & 0.456 \(\pm\) 0.967 & \textbf{1.311} \(\pm\) \textbf{1.177} \\
\midrule
\multirow{6}{*}{CelebA-HQ}
& SmoothGrad & - & - & 0.001 \(\pm\) 0.000 & 0.232 \(\pm\) 0.105 & 0.216 \(\pm\) 0.356 & 0.631 \(\pm\) 0.483 \\
& GuidedBackprop & - & - & 0.001 \(\pm\) 0.000 & 0.385 \(\pm\) 0.097 & 0.222 \(\pm\) 0.313 & 0.475 \(\pm\) 0.434 \\
& GradientSHAP & - & - & 0.011 \(\pm\) 0.004 & 0.061 \(\pm\) 0.114 & 0.188 \(\pm\) 0.305 & 1.107 \(\pm\) 0.663 \\
\cmidrule(lr){2-8}
& Integrated Gradients & 1.000 \(\pm\) 0.000 & \((1.3\pm0.1)\times10^6\) & 0.010 \(\pm\) 0.003 & 0.060 \(\pm\) 0.114 & 0.188 \(\pm\) 0.304 & 1.108 \(\pm\) 0.662 \\
& DDIM & 1.047 \(\pm\) 0.006 & \((2.8\pm0.1)\times10^5\) & \textbf{0.003} \(\pm\) \textbf{0.001} & \textbf{0.147} \(\pm\) \textbf{0.097} & \textbf{0.175} \(\pm\) \textbf{0.321} & \textbf{0.897} \(\pm\) \textbf{0.579} \\
\rowcolor{oursred}\cellcolor{white} & Transport Flow & 1.011 \(\pm\) 0.005 & 1.780 \(\pm\) 0.557 & \textbf{0.003} \(\pm\) \textbf{0.001} & 0.091 \(\pm\) 0.101 & 0.184 \(\pm\) 0.320 & 0.926 \(\pm\) 0.587 \\
\bottomrule
\end{tabular}}
\end{table}

% On CIFAR-10, Transport Flow achieves the lowest SATV
% and blur-deletion score among the evaluated path-based
% methods, while DDIM achieves the lowest zero-deletion
% score. On CelebA-HQ, Transport Flow and DDIM obtain
% similar SATV scores, with DDIM achieving slightly
% better deletion performance.
% Overall, Transport Flow produces spatially coherent attributions with competitive deletion faithfulness among the evaluated path-based methods on both datasets (Table~\ref{tab:main_benchmark}).
% Separately, the synthetic additive sanity check in
% Appendix~\ref{app:additive_sanity_check} shows numerical agreement with analytical
% coordinate contributions. This checks the fixed-path
% attribution estimator rather than the learned
% transport mechanism.

On CIFAR-10, Transport Flow achieves the lowest SATV
and blur-deletion score among the evaluated path-based
methods, while DDIM achieves the lowest zero-deletion
score. On CelebA-HQ, both obtain similar SATV scores,
with DDIM achieving slightly better deletion performance.
Overall, Transport Flow produces spatially coherent attributions with competitive deletion faithfulness among path-based methods on both datasets (Table~\ref{tab:main_benchmark}).
The additive sanity check in
Appendix~\ref{app:additive_sanity_check} shows numerical agreement with analytical
coordinate contributions. This checks the fixed-path
attribution estimator rather than the
transport mechanism.

\begin{figure}[t]
\centering
\includegraphics[width=0.98\linewidth]{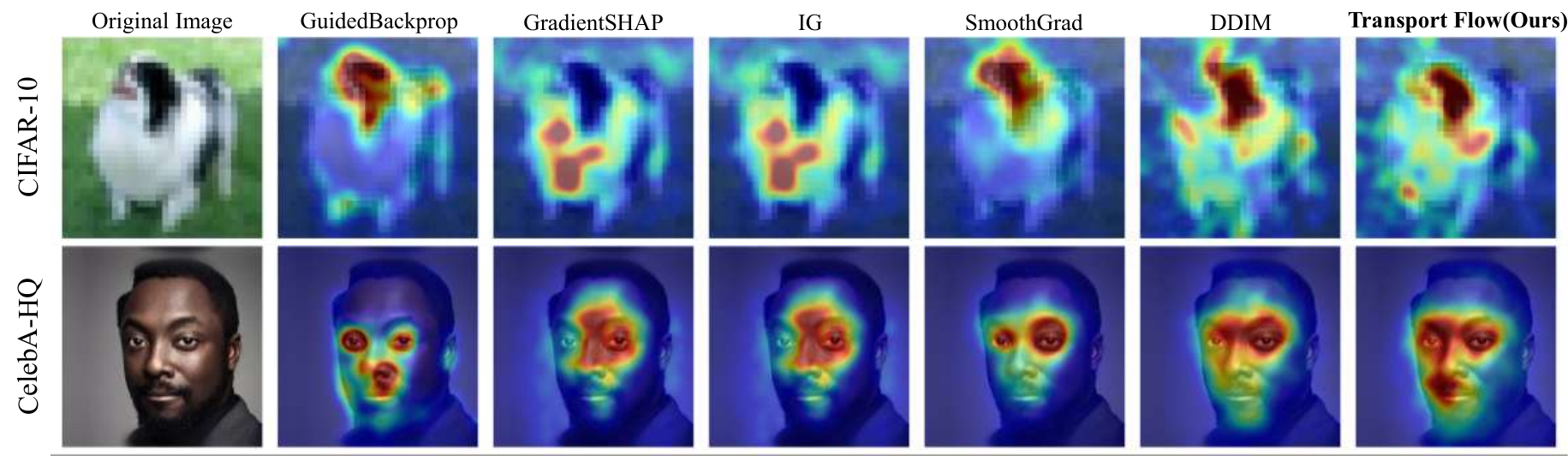}
\caption{Qualitative visualization on CIFAR-10 and CelebA-HQ. Transport Flow explanations are spatially coherent and avoid some high-frequency artifacts of straight-line Integrated Gradients. The figure should be read together with Table~\ref{tab:main_benchmark}; visual clarity alone is not a complete faithfulness metric.}
\label{fig:qual_vis_combined}
\end{figure}

The qualitative examples in Figure~\ref{fig:qual_vis_combined} match the quantitative trend. 
Compared with conventional IG, our method uses an
input-conditioned reference recovered through the learned
transport flow, rather than a fixed reference.
DDIM provides an alternative generative trajectory,
while our transport paths exhibit consistency with
the learned vector field, as reflected in the FCE results
in Table~\ref{tab:main_benchmark}.

To distinguish reference selection from trajectory effects,
we compare the 1-RF characteristic with straight-line IG
using the same endpoint pair $(\hat{x}_0,x_1)$,
where $\hat{x}_0$ is recovered by 1-RF.
On a CIFAR-10 subset, both achieve comparable faithfulness,
with slightly better deletion scores for IG and insertion
scores for 1-RF.
This is consistent with the ideal quadratic OT formulation,
where transport-geodesic attribution coincides with IG
using the optimally coupled reference endpoint.
Full results are provided in
Appendix~\ref{app:same_endpoint_ablation}.

\section{Discussion}
\label{sec:discussion}
% \paragraph{Is least-action transport the only path-selection principle?}
% \label{sec:discussion}

% % \paragraph{Is least-action transport the only path-selection principle?}
% Given a reference distribution and a data distribution, many possible couplings and transport paths remain. We adopt least-action optimal transport as a principled criterion for resolving this ambiguity. Under quadratic Euclidean cost, minimizing kinetic action corresponds to selecting an optimal coupling that minimizes the expected squared displacement. Unlike path-selection methods based on model sensitivities, this criterion depends on the distributional geometry rather than the predictor's gradients, keeping path selection separate from credit allocation.
% We do not claim that least-action transport is universally the best choice for attribution. 
% Rather, it provides a mathematically well-defined and consistent choice, conditional on the specified reference distribution and ground cost. The same-endpoint IG and MIG comparisons in Table~\ref{tab:reference_path_controls} examine the empirical scope of this choice; Appendix~\ref{app:not_on_manifold} explains how they relate to reference coupling and residual learned-trajectory effects.

\noindent\textbf{Scope of least-action transport.}
Least-action transport is not the only possible
principle for path selection. We adopt it as a
mathematically grounded criterion for resolving
coupling ambiguity independently of the predictor's
gradients. De Lara et al.~\citep{JMLR:v25:21-1440}
study transport-based counterfactual models and
establish conditions under which optimal transport
can recover structural counterfactuals, providing
related motivation for distribution-level
counterfactual matching. 
We do not claim that least-action transport is
universally the best choice for attribution.
Rather, it provides a well-defined path-selection
principle conditional on the specified reference
distribution and ground cost, while alternative
transport geometries may be appropriate for
different explanation objectives.

\noindent\textbf{Extensions to other modalities.}
Although our experiments focus on vision, the fixed-path attribution result applies more broadly to differentiable scalar predictors along sufficiently regular paths. Extending the transport component additionally requires meaningful reference and target distributions and a continuous flow. Continuous tabular data is a natural next step, while categorical features may require constrained flows or continuous relaxations. For NLP, a possible extension is to operate in embedding space and map the resulting attributions back to tokens. We leave empirical validation of these extensions to future work.

\section{Conclusion}
\label{sec:conclusion}

We reformulate feature attribution by separating fixed-path credit allocation from path selection. Building on the Aumann–Shapley attribution principle, we introduce transport-geodesic attribution based on least-action optimal transport to select paths between reference and data distributions. Our experiments show that Reflow improves cross-seed attribution stability, while learned transport paths yield spatially coherent explanations with competitive deletion faithfulness among the evaluated path-based methods. Together, these results support distribution-level transport as a principled approach to reference and path selection for feature attribution.

\bibliographystyle{unsrtnat}
\bibliography{references}

\appendix
\clearpage
\section*{Appendix Contents}
\label{app:contents}
\addcontentsline{toc}{section}{Appendix Contents}

\noindent This appendix provides related work, proofs of the main results, optimal transport background, implementation details, evaluation protocols, additional experiments, and limitations.

\vspace{0.8em}
\noindent\hyperref[app:related_work]{Appendix~\ref*{app:related_work}: Related Work}\dotfill\pageref{app:related_work}\\[0.25em]
\noindent\hyperref[app:proofs]{Appendix~\ref*{app:proofs}: Proofs of main results}\dotfill\pageref{app:proofs}\\[0.25em]
\hyperref[app:transport_details]{Appendix~\ref*{app:transport_details}: Optimal transport details and the meaning of transport-geodesic}\dotfill\pageref{app:transport_details}\\[0.25em]
\hyperref[app:implementation_details]{Appendix~\ref*{app:implementation_details}: Implementation details}\dotfill\pageref{app:implementation_details}\\[0.25em]
\hyperref[app:metrics_details]{Appendix~\ref*{app:metrics_details}: Evaluation metrics and protocols}\dotfill\pageref{app:metrics_details}\\[0.25em]
\hyperref[app:additional_results]{Appendix~\ref*{app:additional_results}: Additional results and visualizations}\dotfill\pageref{app:additional_results}\\[0.25em]
\hyperref[app:limitations]{Appendix~\ref*{app:limitations}: Limitations and scope of the claims}\dotfill\pageref{app:limitations}\\[0.25em]

\section{Related Work}
\label{app:related_work}
Feature attribution has several classic lines. Gradient saliency methods explain a prediction by differentiating the output with respect to the input or an internal representation. Early saliency maps, Grad-CAM, SmoothGrad, and Guided Backpropagation are representative examples \citep{simonyan2013deep,selvaraju2017grad,DBLP:journals/corr/SmilkovTKVW17,springenberg2014striving}. These methods are efficient, but a local gradient does not by itself describe the finite change from a reference state to the input. Perturbation methods instead mask or modify input regions and measure the output change. Occlusion, meaningful perturbation, LIME, RISE, and SHAP-style methods belong to this family \citep{zeiler2014visualizing,fong2017interpretable,ribeiro2016should,petsiuk2018rise,lundberg2017unified}. They connect more directly to counterfactual reasoning, but they must choose how to replace missing content. Backpropagation decomposition methods such as Layer-wise Relevance Propagation, Deep Taylor decomposition, and DeepLIFT propagate relevance through network layers using conservation or difference rules \citep{bach2015pixel,montavon2017explaining,shrikumar2017learning}. These methods motivate conservation of total relevance, but they do not solve the global path-selection problem studied here.

Shapley-value explanations inherit their appeal from cooperative game theory \citep{shapley1953value}. The Shapley value is unique for a finite game once the coalition value function is fixed. In model explanation, however, the coalition value function is not fixed by the predictor alone. Interventional, conditional, causal, and asymmetric Shapley variants encode different assumptions about feature absence and feature dependence \citep{strumbelj2010efficient,lundberg2017unified,sundararajan2020many,frye2020asymmetric}. Manifold-restricted Shapley methods use generative or conditional models to keep perturbations closer to the data distribution \citep{frye2020shapley,taufiq2023manifold}. These works identify the missing-feature problem. We take a different route. We do not define all subset values. We formulate a continuous path game and ask how to choose the path that reveals the input.

Path attribution methods show that continuous paths can replace discrete coalitions. Integrated Gradients is the most widely used example and can be viewed as an Aumann-Shapley line integral along the straight path from a baseline to the input \citep{sundararajan2017axiomatic,aumann2015values}. Expected Gradients averages Integrated Gradients over a background distribution \citep{erion2021improving}. Guided Integrated Gradients and related adaptive methods alter the route in input space to reduce visual artifacts \citep{kapishnikov2021guided}. Recent geometric variants make the role of the path even more explicit. Manifold Integrated Gradients and Geodesic Integrated Gradients replace the straight path with a Riemannian geodesic under a learned data geometry or a model-induced geometry \citep{zaher2024manifold,salek2025using}. We view these methods as close and complementary, not as targets that our method must dominate. Their path-selection object is an instance-level geodesic in a prescribed
input-space metric, and in Geodesic Integrated Gradients this metric is induced
by the explained model. Our path-selection object is different: it is a
characteristic curve of a distribution-level least-action transport process
from \(p_0\) to \(p_1\). Thus, we separate the geometry that chooses the path
from the classifier whose score is explained. The generative transport model
chooses the path, and the classifier gradient only allocates score change along
that path. This changes the question from how to choose a better curve between two endpoints to how to select a generative transport process between two distributions.

Generative models provide structured transitions from simple priors to data-like samples. Diffusion models and DDIM define reverse-time generative trajectories, while continuous normalizing flows and flow matching define time-dependent vector fields \citep{song2020denoising,lipman2022flow}. Rectified Flow learns straighter trajectories by regressing velocities between coupled reference and data samples, and Reflow further rectifies the learned coupling \citep{liu2022flow}. These models are useful for attribution because they provide more structured trajectories than pixel-space interpolation. Still, a generative path is not automatically a principled explanation path. We use optimal transport to specify which generative path we want.

Optimal transport gives a geometric way to compare probability distributions \citep{villani2008optimal}. In the dynamic Benamou-Brenier formulation, the squared Wasserstein-2 distance is the minimum kinetic action among all density paths and velocity fields that move one distribution to another \citep{benamou2000computational}. This formulation matches our attribution problem. If explaining an input means measuring the model score along a transition from a reference distribution \(p_0\) to a data distribution \(p_1\), then the transition path should not be arbitrary. We use the Wasserstein-2 geodesic as the least-action path-selection principle.

Evaluation of attributions remains difficult. Deletion, insertion, faithfulness correlation, and sanity checks measure different aspects of explanation quality \citep{samek2016evaluating,hooker2019benchmark,tomsett2020sanity}. These metrics can favor noisy maps or penalize smooth maps depending on the perturbation protocol. We therefore report standard faithfulness metrics together with path-geometry diagnostics. Our Flow Consistency Error measures consistency with the learned vector field. It is not a certificate of true manifold membership.

\section{Proofs of main results}
\label{app:proofs}

This section proves the fixed-path uniqueness, transport-geodesic characterization, and flow-approximation
stability results. It also establishes the discrete efficiency bound and the attribution formula for additive scores.

\begin{lemma}[Coordinate-trace representation]
\label{lem:trace_representation}
Fix a path \(\gamma\) and a coordinate \(i\). Under the coordinate-trace determined part of Assumption~\ref{ass:fixed_path_axioms}, there is a functional \(L_i\) on the set of admissible scalar traces
\[
\mathcal{H}_i(\gamma)=\left\{t\mapsto \frac{\partial f(\gamma(t))}{\partial x_i}\frac{\mathrm{d}\gamma_i(t)}{\mathrm{d}t}: f\in\cF\right\}
\]
such that
\[
A_i(f,\gamma)=L_i\left(t\mapsto \frac{\partial f(\gamma(t))}{\partial x_i}\frac{\mathrm{d}\gamma_i(t)}{\mathrm{d}t}\right).
\]
Moreover, if the attribution rule is linear in \(f\), then \(L_i\) is linear on \(\mathcal{H}_i(\gamma)\). If the attribution rule is continuous under uniform convergence of coordinate traces, then \(L_i\) is continuous for the uniform norm.
\end{lemma}

\begin{proof}
The coordinate-trace determined condition says exactly that two admissible scores with the same coordinate trace for coordinate \(i\) must receive the same coordinate-\(i\) attribution. Hence we can define \(L_i(h)\) as \(A_i(f,\gamma)\) for any \(f\in\cF\) whose coordinate trace equals \(h\). This definition does not depend on which such \(f\) we choose. If \(h_f\) and \(h_g\) are traces generated by \(f\) and \(g\), then the trace generated by \(af+bg\) is \(ah_f+bh_g\). Linearity of \(A_i\) in the score gives
\[
L_i(ah_f+bh_g)=A_i(af+bg,\gamma)=aA_i(f,\gamma)+bA_i(g,\gamma)=aL_i(h_f)+bL_i(h_g).
\]
This proves linearity of \(L_i\). The continuity statement follows directly from the continuity clause in Assumption~\ref{ass:fixed_path_axioms}.
\end{proof}

\begin{proof}[Proof of Theorem~\ref{thm:fixed_path_uniqueness}]
For each coordinate \(i\), Lemma~\ref{lem:trace_representation} gives a linear functional \(L_i\) on the coordinate-trace space. We compare it with the ordinary integral functional \(I(h)=\int_0^1 h(t)\dd t\). Let \(h\in\mathcal{H}_i(\gamma)\) be any admissible coordinate trace. By Assumption~\ref{ass:trace_richness}, there is an admissible score \(g\) whose coordinate-\(i\) trace is \(h\) and whose coordinate-\(j\) trace is zero for every \(j\ne i\). The dummy property gives \(A_j(g,\gamma)=L_j(0)=0\) for every \(j\ne i\). Efficiency applied to \(g\) gives
\[
L_i(h)=\sum_{j=1}^d A_j(g,\gamma)=g(\gamma(1))-g(\gamma(0)).
\]
The chain rule along \(\gamma\) gives
\[
g(\gamma(1))-g(\gamma(0))
=\int_0^1\nabla g(\gamma(t))^\top\frac{\mathrm{d}\gamma(t)}{\mathrm{d}t}\dd t
=\sum_{j=1}^d\int_0^1\frac{\partial g(\gamma(t))}{\partial x_j}\frac{\mathrm{d}\gamma_j(t)}{\mathrm{d}t}\dd t.
\]
All terms except the \(i\)-th term vanish by the trace separation property, so
\[
g(\gamma(1))-g(\gamma(0))=\int_0^1 h(t)\dd t.
\]
Therefore \(L_i(h)=I(h)\) for every admissible coordinate trace \(h\). Applying this identity to the trace \(h_i^f(t)=\frac{\partial f(\gamma(t))}{\partial x_i}\frac{\mathrm{d}\gamma_i(t)}{\mathrm{d}t}\) of an arbitrary admissible score \(f\), we obtain
\[
A_i(f,\gamma)=L_i(h_i^f)=\int_0^1\frac{\partial f(\gamma(t))}{\partial x_i}\frac{\mathrm{d}\gamma_i(t)}{\mathrm{d}t}\dd t=\Phi_i(f,\gamma).
\]
This proves uniqueness.
\end{proof}

\begin{proof}[Proof of Theorem~\ref{thm:geodesic_characterization}]
Assumption~\ref{ass:transport_regular} gives a unique optimal
dynamic plan from \(p_0\) to \(p_1\), determining the
distributional path and its characteristic curves under the
stated regularity conditions.
Every rule in the specified class therefore selects the same
characteristic \(\gamma_{\vx_1}^\star\).
This path is regular by the theorem's hypothesis, so
Theorem~\ref{thm:fixed_path_uniqueness} applies.
Hence the coordinate attribution must equal
\[
\int_0^1
\frac{\partial f_c(\gamma_{\vx_1}^\star(t))}{\partial x_i}
\frac{\mathrm{d}\gamma_{\vx_1,i}^\star(t)}{\mathrm{d}t}
\dd t,
\]
which is Eq.~\eqref{eq:geodesic_as}.
Conversely, the rule in Eq.~\eqref{eq:geodesic_as} selects
this optimal transport path and applies the Aumann-Shapley
fixed-path rule, so it belongs to the stated class and
satisfies the fixed-path axioms.
This proves the characterization.
\end{proof}

\begin{proof}[Proof of Theorem~\ref{thm:stability}]
Let \(\gamma(t)\) denote the ideal backward characteristic \(\gamma_{\vx_1}^\star(t)\), and let \(\hat{\gamma}(t)\) denote the learned backward characteristic with the same terminal endpoint \(\hat{\gamma}(1)=\gamma(1)=\vx_1\). Write \(L_v\) for a common Lipschitz constant of \(\vv_t^\star\) and \(\hat{\vv}_t\) in \(\vx\) on the compact set \(\cK\). Write \(B_g=\sup_{\vx\in\cK}\|\nabla f_c(\vx)\|_2\), \(L_g\) for the Lipschitz constant of \(\nabla f_c\) on \(\cK\), and \(B_v=\sup_{t\in[0,1],\,\vx\in\cK}
\max\{\|\vv_t^\star(\vx)\|_2,\|\hat{\vv}_t(\vx)\|_2\}\). These constants are finite by Assumption~\ref{ass:stability_regular} and compactness. We write \(\vv_{t,i}^\star(\vx)\) and \(\hat{\vv}_{t,i}(\vx)\) for the \(i\)-th coordinates of the ideal and learned vector fields.

For \(t\le 1\), the two terminal-value ODEs give
\[
\gamma(t)=\vx_1-\int_t^1\vv_s^\star(\gamma(s))\dd s,
\qquad
\hat{\gamma}(t)=\vx_1-\int_t^1\hat{\vv}_s(\hat{\gamma}(s))\dd s.
\]
Taking the difference and using the triangle inequality gives
\[
\|\gamma(t)-\hat{\gamma}(t)\|_2
\le \int_t^1\|\vv_s^\star(\gamma(s))-\vv_s^\star(\hat{\gamma}(s))\|_2\dd s
+\int_t^1\|\vv_s^\star(\hat{\gamma}(s))-\hat{\vv}_s(\hat{\gamma}(s))\|_2\dd s.
\]
The Lipschitz condition and the uniform vector-field error bound imply
\[
\|\gamma(t)-\hat{\gamma}(t)\|_2
\le \int_t^1 L_v\|\gamma(s)-\hat{\gamma}(s)\|_2\dd s+(1-t)\eps.
\]
The backward form of Gronwall's inequality gives
\[
\sup_{t\in[0,1]}\|\gamma(t)-\hat{\gamma}(t)\|_2\le e^{L_v}\eps.
\]
This is the only place where we use an ODE stability result. It says that a uniformly small error in the vector field produces a uniformly small error in the characteristic curve over a finite time interval.

For a fixed coordinate \(i\), subtract the two attribution integrals:
\[
|\Psi_i-\hat{\Psi}_i|
\le \int_0^1\left|
\frac{\partial f_c(\gamma(t))}{\partial x_i}\vv_{t,i}^\star(\gamma(t))
-
\frac{\partial f_c(\hat{\gamma}(t))}{\partial x_i}\hat{\vv}_{t,i}(\hat{\gamma}(t))
\right|\dd t.
\]
The integrand is bounded by
\[
B_g\|\vv_t^\star(\gamma(t))-\hat{\vv}_t(\hat{\gamma}(t))\|_2
+B_v\|\nabla f_c(\gamma(t))-\nabla f_c(\hat{\gamma}(t))\|_2.
\]
Using Lipschitz continuity and the vector-field error bound, we have
\[
\|\vv_t^\star(\gamma(t))-\hat{\vv}_t(\hat{\gamma}(t))\|_2
\le L_v\|\gamma(t)-\hat{\gamma}(t)\|_2+\eps
\le (L_v e^{L_v}+1)\eps.
\]
We also have
\[
\|\nabla f_c(\gamma(t))-\nabla f_c(\hat{\gamma}(t))\|_2
\le L_g\|\gamma(t)-\hat{\gamma}(t)\|_2
\le L_g e^{L_v}\eps.
\]
Combining the last three displays and integrating over a time interval of length one gives
\[
|\Psi_i-\hat{\Psi}_i|
\le \left[B_g(L_v e^{L_v}+1)+B_vL_g e^{L_v}\right]\eps.
\]
The constant in brackets depends only on \(\cK\), the
time horizon, the uniform bounds and spatial Lipschitz constants of the vector fields, and the gradient bound and gradient Lipschitz constant of \(f_c\). This proves Eq.~\eqref{eq:stability_bound}.
\end{proof}

\begin{proposition}[Discrete efficiency residual]
\label{prop:discrete_efficiency}
Let \(f\) have Hessian norm bounded by \(M\) on a compact set containing every line segment
\([\hat{\vx}_k,\hat{\vx}_{k+1}]\), \(k=0,\ldots,K-1\). Let \(\Delta\hat{\vx}_k=\hat{\vx}_{k+1}-\hat{\vx}_k\), and define the discrete attribution by Eq.~\eqref{eq:discrete_estimator}. Then
\[
\left|\sum_{i=1}^d\hat{\Psi}_i-\left(f(\hat{\vx}_K)-f(\hat{\vx}_0)\right)\right|
\le \frac{M}{2}\sum_{k=0}^{K-1}\|\Delta\hat{\vx}_k\|_2^2.
\]
If the path increments satisfy \(\|\Delta\hat{\vx}_k\|_2\le V/K\), then the residual is at most \(MV^2/(2K)\).
\end{proposition}

\begin{proof}
Taylor's theorem with remainder gives, for each segment of the discrete path,
\[
f(\hat{\vx}_{k+1})=f(\hat{\vx}_k)+\nabla f(\hat{\vx}_k)^\top\Delta\hat{\vx}_k+r_k,
\qquad |r_k|\le \frac{M}{2}\|\Delta\hat{\vx}_k\|_2^2.
\]
Summing this identity from \(k=0\) to \(K-1\) makes the left side telescope:
\[
f(\hat{\vx}_K)-f(\hat{\vx}_0)=\sum_{k=0}^{K-1}\nabla f(\hat{\vx}_k)^\top\Delta\hat{\vx}_k+\sum_{k=0}^{K-1}r_k.
\]
The first sum is exactly the sum over coordinates of Eq.~\eqref{eq:discrete_estimator}. Taking absolute values and applying the remainder bound gives the first claim. If \(\|\Delta\hat{\vx}_k\|_2\le V/K\), then \(\sum_k\|\Delta\hat{\vx}_k\|_2^2\le K(V/K)^2=V^2/K\), which gives the second claim.
\end{proof}

\begin{proposition}[Additive scores]
\label{prop:additive_scores}
Suppose \(f(\vx)=b+\sum_{i=1}^d f_i(x_i)\) and the path \(\gamma\) connects \(\vx_0\) to \(\vx_1\). Then the Aumann-Shapley attribution along any continuously differentiable path is
\[
\Phi_i(f,\gamma)=f_i(x_{1,i})-f_i(x_{0,i}).
\]
\end{proposition}

\begin{proof}
For an additive score, \(\partial_i f(\vx)=f_i'(x_i)\). Therefore
\[
\Phi_i(f,\gamma)=\int_0^1 f_i'(\gamma_i(t))\frac{\mathrm{d}\gamma_i(t)}{\mathrm{d}t}\dd t.
\]
The one-dimensional chain rule gives \(\frac{\mathrm{d}}{\mathrm{d}t} f_i(\gamma_i(t))=f_i'(\gamma_i(t))\frac{\mathrm{d}\gamma_i(t)}{\mathrm{d}t}\). Hence
\[
\Phi_i(f,\gamma)=\int_0^1\frac{\mathrm{d}}{\mathrm{d}t} f_i(\gamma_i(t))\dd t=f_i(\gamma_i(1))-f_i(\gamma_i(0)).
\]
Since \(\gamma_i(1)=x_{1,i}\) and \(\gamma_i(0)=x_{0,i}\), the claim follows.
\end{proof}

\section{Optimal transport details and the meaning of transport-geodesic}
\label{app:transport_details}

This section connects optimal endpoint couplings, displacement interpolation, and kinetic action, and explains their roles in path attribution.

\subsection{From static couplings to dynamic paths}
\label{app:static_to_dynamic}

A coupling between \(p_0\) and \(p_1\) is a joint law
\(\pi\) on pairs \((\vx_0,\vx_1)\) whose first marginal
is \(p_0\) and whose second marginal is \(p_1\).
Let \(\Pi(p_0,p_1)\) denote the set of all such couplings,
and let \(W_2(p_0,p_1)\) denote the quadratic Wasserstein
distance. The quadratic-cost optimal transport problem is
\[
W_2^2(p_0,p_1)
=
\inf_{\pi\in\Pi(p_0,p_1)}
\int_{\R^d\times\R^d}
\|\vx_1-\vx_0\|_2^2
\dd\pi(\vx_0,\vx_1).
\]
An optimal coupling \(\pi^\star\) induces the displacement
interpolation
\[
\vx_t=(1-t)\vx_0+t\vx_1,
\qquad
(\vx_0,\vx_1)\sim\pi^\star.
\]
The law of \(\vx_t\) is a distribution \(\rho_t\).
Under standard regularity conditions, this path of laws
is the constant-speed Wasserstein geodesic.
In the deterministic Monge case, where
\(\vx_1=T(\vx_0)\) for an optimal transport map \(T\),
each sample moves along a straight segment from
\(\vx_0\) to \(T(\vx_0)\).

\subsection{Kinetic action and the Benamou-Brenier view}
\label{app:benamou_brenier_details}

The dynamic formulation used in Eq.~\eqref{eq:benamou_brenier} replaces couplings by probability-law paths and velocity fields. A pair \((\rho,\vv)\), with \(\rho=(\rho_t)_{t\in[0,1]}\) and \(\vv=(\vv_t)_{t\in[0,1]}\), is admissible when it satisfies the continuity equation in Eq.~\eqref{eq:continuity}. This equation means that probability mass moves according to the velocity field and is neither created nor destroyed. The action
\[
\cA(\rho,\vv)=\int_0^1\int_{\R^d}\|\vv_t(\vx)\|_2^2\dd\rho_t(\vx)\dd t
\]
measures the average squared speed of the mass over time. The Benamou-Brenier formula says that the minimum action over all admissible pairs equals \(W_2^2(p_0,p_1)\) \citep{benamou2000computational,villani2008optimal}. This is why we call the selected path a transport-geodesic path: it is the least-action path between distributions in the Wasserstein geometry.

\begin{proposition}[Action of an optimal displacement interpolation]
\label{prop:action_displacement}
Assume that the optimal coupling is induced by a map \(T\), so \(\vx_1=T(\vx_0)\) and \(\vx_0\sim p_0\). Define \(\vx_t=(1-t)\vx_0+tT(\vx_0)\). Then the associated constant velocity along each particle is \(T(\vx_0)-\vx_0\), and the kinetic action equals the quadratic transport cost:
\[
\int_0^1\E\|T(\vx_0)-\vx_0\|_2^2\dd t
=\E\|T(\vx_0)-\vx_0\|_2^2.
\]
If \(T\) is optimal, this value equals \(W_2^2(p_0,p_1)\).
\end{proposition}

\begin{proof}
For each fixed starting point \(\vx_0\), differentiating \(\vx_t=(1-t)\vx_0+tT(\vx_0)\) with respect to \(t\) gives \(\mathrm{d}\vx_t/\mathrm{d}t=T(\vx_0)-\vx_0\). This velocity is constant in time along that particle. Therefore the particle action over \([0,1]\) is
\[
\int_0^1\left\|\frac{\mathrm{d}\vx_t}{\mathrm{d}t}\right\|_2^2\dd t
=\int_0^1\|T(\vx_0)-\vx_0\|_2^2\dd t
=\|T(\vx_0)-\vx_0\|_2^2.
\]
Taking expectation over \(\vx_0\sim p_0\) gives the displayed identity. If \(T\) is the optimal quadratic-cost transport map, the expected squared displacement is the definition of \(W_2^2(p_0,p_1)\) in the Monge formulation.
\end{proof}

An ideal \(W_2\) geodesic combines an optimal endpoint coupling with constant-speed straight displacement paths. Straightness describes the individual trajectories, whereas transport optimality also depends on the coupling. We use kinetic-action and oracle-based diagnostics to assess learned transport paths.

\subsection{Transport geometry and reference selection}
\label{app:not_on_manifold}

The phrase ``on-manifold'' can mean different things. A strict version would assume a low-dimensional set \(\mathcal{M}\subset\R^d\), a data distribution supported on \(\mathcal{M}\), and a certificate that every path point \(\gamma(t)\) belongs to \(\mathcal{M}\). We do not make this assumption and we do not prove such a certificate.

Our statement is distributional. We choose a flow whose time marginals \(\rho_t\) move from \(p_0\) to \(p_1\). At intermediate times, \(\rho_t\) is generally neither \(p_0\) nor \(p_1\). If \(p_0\) is a Gaussian prior and \(p_1\) is an image distribution, intermediate states are generated states along a transport bridge. They may look more structured than arbitrary pixel interpolation, but this visual fact is not a mathematical proof of data-manifold membership.

For this reason, the main text uses the terms transport-consistent and generative. Flow Consistency Error checks whether a numerical path follows the learned vector field. It does not check whether the true data density is high at every intermediate point.

\paragraph{Separating reference selection from trajectory effects.}
The same-endpoint ablation in
Table~\ref{tab:same_endpoint_ablation} compares the learned
1-RF characteristic with the straight path
\[
\gamma_{\mathrm{line}}(t)=(1-t)\hat{x}_0+t x_1,
\]
where \(\hat{x}_0\) is recovered by the same 1-RF model.
The input, recovered reference, integration step count, and
attribution and evaluation protocols are held fixed between
these two methods. The unmatched-reference IG row provides
an additional reference-selection comparison.
The point estimates favor same-endpoint IG on both deletion
metrics and the learned characteristic on insertion.
Thus, this experiment identifies no systematic faithfulness
advantage from the residual curvature of the learned 1-RF
trajectory across the reported metrics. It isolates trajectory
effects for a given learned reference; optimality of the
reference coupling remains a separate question.

\paragraph{The role of quadratic transport.}
As described in Appendices~C.1--C.2, ideal quadratic transport
selects an optimal endpoint coupling and connects each matched
pair by straight displacement interpolation. For an optimally
matched pair \((x_0^\star,x_1)\),
\[
\gamma^\star_{x_1}(t)=(1-t)x_0^\star+t x_1.
\]
Substituting this path into the attribution integral gives
\[
\begin{aligned}
\Psi_i(f_c,x_1)
&=
(x_{1,i}-x_{0,i}^\star)
\int_0^1
\partial_i f_c\bigl((1-t)x_0^\star+t x_1\bigr)
\,\mathrm{d}t \\
&=
\mathrm{IG}_i(f_c;x_0^\star,x_1).
\end{aligned}
\]
Consequently, the ideal construction coincides with Integrated
Gradients from an optimally coupled, input-conditioned reference.
Its nontrivial transport contribution is the distribution-level
coupling and reference selection. Rectified Flow and Reflow
provide scalable approximations to this endpoint-conditioned
transport. Differences between a learned characteristic and
same-endpoint IG reflect residual trajectory geometry together
with finite ODE and quadrature effects.

\subsection{Coordinate-trace separability and self-intersecting paths.}
Assumption~\ref{ass:trace_richness} requires separation only of realizable coordinate traces. Regularity of a path does not exclude self-intersections. If \(\gamma(t_a)=\gamma(t_b)\) for \(t_a\ne t_b\), every \(C^1\) score \(g\) must satisfy
\[
\nabla g(\gamma(t_a))=\nabla g(\gamma(t_b)).
\]
Consistency of values and gradients at repeated spatial points can prevent coordinate-trace separation, even when the trace to be retained is realizable. Separability therefore remains an explicit assumption for the fixed-path uniqueness characterization. The Aumann--Shapley integral itself remains well defined for \(C^1\) scores along \(C^1\) paths, including self-intersecting paths.

For smooth embedded paths, separability follows from extending compatible first-order information along the curve when \(\cF\) contains the required extensions. This includes the nondegenerate straight displacement paths
\[
\gamma(t)=(1-t)x_0+tx_1,\qquad x_0\ne x_1,
\]
arising under ideal quadratic optimal transport, subject to the same score-class requirement.

\subsection{Endpoint conditioning}
\label{app:endpoint_conditioning}

The ideal definition in Eq.~\eqref{eq:geodesic_ode} conditions on a target endpoint \(\vx_1\). In an exact deterministic transport map, this endpoint has a unique preimage \(\vx_0\) under the flow map, except on sets where the map is not invertible. In a learned ODE model, we approximate this preimage by integrating the learned vector field backward from \(\vx_1\) to time zero. We then integrate forward from the obtained \(\hat{\vx}_0\) to store a stable numerical trajectory. This backward-forward procedure keeps the endpoint tied to the input being explained and avoids sampling an unrelated reference point.

Numerical integration errors can cause the backward-forward trajectory to end at a point different from \(\vx_1\). We set \(\hat{\vx}_K\leftarrow\vx_1\) before accumulating attributions. The resulting stored path connects the recovered reference \(\hat{\vx}_0\) to the observed input \(\vx_1\). The completeness residual in Table~\ref{tab:completeness} reports the remaining numerical integration error.

\section{Implementation details}
\label{app:implementation_details}

This section describes how we instantiate the ideal transport-geodesic attribution with a learned flow. The main point is that all baselines are evaluated through the same path-integral form whenever they define a path. This keeps the comparison focused on the path rather than on a different allocation formula.

\subsection{Learned vector field and Reflow}
\label{app:rf_reflow_details}

Rectified Flow fits a time-dependent vector field
\(\vv_{\vtheta}(\vx,t)\) using
Eq.~\eqref{eq:rectified_flow_loss}.
We write \(\hat{\vv}_t(\vx)=\vv_{\vtheta}(\vx,t)\)
for the learned field.
The training pairs \((\vz_0,\vz_1)\) are drawn from
a coupling \(\pi\), typically the independent coupling
between the reference prior and data distribution for 1-RF.
For Reflow, the learned flow maps prior samples to generated
endpoints, and the induced endpoint pairs are used to train
a new vector field.
This procedure refines the training coupling and is used
to straighten transport trajectories \citep{liu2022flow}.

\paragraph{Experimental setup.}
We use ResNet-18 predictors with approximately \(57\%\)
accuracy on CUB-200-2011 at \(32\times32\),
approximately \(98\%\) on CIFAR-10 at \(32\times32\),
and above \(98\%\) on the CelebA-HQ gender attribute
task at \(256\times256\).
For CUB-200-2011, generative models are fine-tuned
from pretrained checkpoints.
The main CIFAR-10 and CelebA-HQ image benchmarks use
official pretrained RF and diffusion checkpoints.
The numerical completeness study uses 2-RF, while
the main image benchmarks use 1-RF. The theoretical guarantees apply under the stated regularity assumptions, including the smoothness conditions on the target score. We evaluate completeness and attribution stability empirically for the experimental predictors.

\paragraph{Implementation and sample selection.}
Experiments are implemented in PyTorch on a single
NVIDIA RTX 3090 GPU with 24 GB memory.
Flow trajectories are integrated with the
\texttt{rk45} solver, and the default attribution grid
uses \(K=50\) intervals.
For classifier-based evaluations, we select correctly
classified samples with prediction confidence above
\(0.6\) for the axiomatic and stability studies and
above \(0.8\) for the image benchmarks.

\subsection{Reference endpoint for a given input}
\label{app:reference_endpoint}

For an observed input \(\vx_1\), we need a reference endpoint tied to this input. We obtain it by solving the learned ODE backward from time one to time zero:
\[
\frac{\mathrm{d}\vx_t}{\mathrm{d}t}=\vv_{\vtheta}(\vx_t,t),
\qquad \vx_{t=1}=\vx_1.
\]
The result is \(\hat{\vx}_0\). We then integrate forward from \(\hat{\vx}_0\) and store the states \(\hat{\vx}_0,\ldots,\hat{\vx}_K\) on the uniform grid \(t_k=k/K\), with \(\Delta t=1/K\). The backward solve ties the recovered reference to the observed input through the learned vector field. This differs from Expected Gradients or GradientSHAP, where the reference is sampled independently from a background distribution.

Finite solver accuracy can produce a terminal mismatch after the backward-forward integration. We set \(\hat{\vx}_K\leftarrow\vx_1\) before evaluating Eq.~\eqref{eq:discrete_estimator}. Completeness is evaluated against \(f_c(\vx_1)-f_c(\hat{\vx}_0)\) using this endpoint-corrected stored trajectory.

The discrete kinetic action, GPS, and FCE are evaluated using the same endpoint-corrected stored states.

\subsection{Numerical integration}
\label{app:numerical_integration_details}

The discrete estimator in Eq.~\eqref{eq:discrete_estimator}
is a left Riemann estimator for the path integral
on the grid \(t_k=k/K\).
We use it because it is simple and matches the
finite-difference view of adding small path increments.
For sufficiently smooth scores and paths, midpoint
quadrature can achieve second-order accuracy by
evaluating gradients at intermediate states.
We compare both quadrature rules in
Appendix~\ref{app:additive_sanity_check}.
Proposition~\ref{prop:discrete_efficiency}
shows that the completeness residual is controlled
by the squared path increments when the predictor
has bounded Hessian on the visited region.

In practice, one explanation with \(K\) integration
intervals requires storing \(K+1\) path states.
Both left Riemann and midpoint quadrature require
\(K\) target-score gradient evaluations, although
midpoint quadrature requires intermediate path states.
We use the target logit before the softmax rather
than the probability after the softmax.
Logits avoid saturation effects and are standard
in attribution experiments.
We aggregate pixel-level RGB attributions by summing
or taking the channelwise absolute value depending
on the visualization protocol.
For deletion metrics, we use the signed attribution
score to rank features when the target is a score
increase and use absolute scores only when the
baseline method is defined as unsigned.

\subsection{Baselines}
\label{app:baseline_details}

We compare against local gradient smoothing methods, backpropagation decomposition methods, Shapley-style randomized baselines, straight-line path methods, and generative path methods. SmoothGrad averages gradients under small input noise. Guided Backpropagation modifies the backward pass through ReLU layers. GradientSHAP samples references and interpolation coefficients. Integrated Gradients uses the straight path from a fixed reference to the input. DDIM supplies a deterministic generative trajectory from a diffusion model. Our method uses the learned Rectified Flow or Reflow trajectory and then applies the same Aumann-Shapley coordinate integral.

GPS and FCE are reported only for methods with a
specified continuous path; the other entries in
Table~\ref{tab:main_benchmark} are marked with a dash.
FCE is evaluated against our learned RF field and
serves as a dynamics diagnostic, separately from
the perturbation-based faithfulness scores.

\subsection{Computational cost}
\label{app:cost_details}

Exact discrete Shapley computation for a general feature-level game requires an exponential number of coalition evaluations. Practical explainers reduce this cost through coalition sampling or feature grouping. 

% The main cost is the repeated gradient evaluation of the fixed predictor along the path. The generative ODE cost is separate and depends on the solver and the vector field architecture.

For Eq.~\eqref{eq:discrete_estimator}, attribution accumulation requires \(K\) predictor gradient evaluations and scales linearly with the number
of attribution intervals. Constructing the learned path adds backward and
forward ODE integration. The ODE function-evaluation count is solver-dependent and is distinct from the attribution grid size \(K\). Overall computational cost also depends on the
input resolution and the predictor and vector-field architectures. Compared with straight-line IG at matched gradient evaluation counts, the additional work is the generative path construction.

\section{Evaluation metrics and protocols}
\label{app:metrics_details}

Attribution metrics measure different properties, and no single metric proves explanation quality. We therefore separate four questions. First, we check whether a sampled path is geometrically short in the ambient space. Second, we check whether the sampled path follows the learned vector field. Third, we check whether the attribution map is visually structured. Fourth, we check whether important pixels affect the target score under deletion. Throughout this section, the discrete path is \(\hat{\vx}_0,\ldots,\hat{\vx}_K\), the grid is uniform with \(t_k=k/K\), and \(\Delta t=1/K\). We write \(\hat{\vv}_{t_k}(\hat{\vx}_k)=\vv_{\vtheta}(\hat{\vx}_k,t_k)\) for the learned velocity evaluated at the sampled state.

\subsection{Geometric path straightness}
\label{app:gps_details}

For a discrete path \(\hat{\vx}_0,\ldots,\hat{\vx}_K\), the Geometric Path Straightness score is
\[
\mathrm{GPS}=\frac{\sum_{k=0}^{K-1}\|\hat{\vx}_{k+1}-\hat{\vx}_k\|_2}{\|\hat{\vx}_K-\hat{\vx}_0\|_2}.
\]
The denominator is the endpoint displacement and the numerator is the discrete path length. The score is at least one by the triangle inequality whenever \(\hat{\vx}_K\ne\hat{\vx}_0\). A value close to one means that the sampled path is close to a straight segment in ambient Euclidean space. This does not by itself imply optimal transport. Integrated Gradients has GPS exactly one for a straight path. GPS characterizes trajectory geometry, while
transport optimality additionally depends on the endpoint coupling.

\subsection{Flow Consistency Error}
\label{app:fce_details}

The Flow Consistency Error is
\[
\mathrm{FCE}=\frac{1}{K}\sum_{k=0}^{K-1}\left\|\frac{\hat{\vx}_{k+1}-\hat{\vx}_k}{\Delta t}-\hat{\vv}_{t_k}(\hat{\vx}_k)\right\|_2^2.
\]
This metric compares the finite-difference velocity of
the sampled path with the velocity predicted by the
learned flow. Low FCE indicates agreement with the
learned vector-field family \(\hat{\vv}_t\).
Since Transport Flow uses this same field to generate
its paths, FCE serves as an internal dynamics diagnostic,
not as a faithfulness metric or a basis for ranking
attribution methods.

For methods without a specified continuous path,
such as SmoothGrad or Guided Backpropagation,
FCE is not applicable.
For IG and DDIM, it measures the mismatch between
their paths and the same learned RF field.

\subsection{Structure-aware total variation}
\label{app:satv_details}

Let \(\Omega\) be the two-dimensional pixel grid of an image. Let \(I\) denote the input image, and let \(\vphi\) denote the attribution map after summing or averaging over color channels. We write \(\nabla_{\mathrm{img}}\) for finite differences on the image grid. For a pixel \((u,v)\in\Omega\), \(\nabla_{\mathrm{img}}\vphi_{u,v}\) is the local finite-difference gradient of the attribution map and \(\nabla_{\mathrm{img}} I_{u,v}\) is the corresponding image gradient. The Structure-Aware Total Variation diagnostic is
\[
\mathrm{SATV}(\vphi)=
\sum_{(u,v)\in\Omega}
\|\nabla_{\mathrm{img}}\vphi_{u,v}\|_1
\exp\left(-\alpha\|\nabla_{\mathrm{img}} I_{u,v}\|_2\right),
\]
where \(\alpha>0\) controls how strongly the penalty is reduced near image edges. We set \(\alpha=10\) in all experiments. The exponential weight reduces the penalty near image edges and keeps the penalty high in flat regions. As a result, SATV penalizes high-frequency attribution noise where the image itself has little structure. SATV summarizes spatial structure and is interpreted alongside deletion metrics and completeness checks. Completeness is evaluated using the original signed, unnormalized coordinate attributions.

\subsection{Edge alignment score}
\label{app:eas_details}

The Edge Alignment Score measures whether changes in the attribution map occur near image edges. We define image edge strength and attribution edge strength by
\[
e_{u,v}=\|\nabla_{\mathrm{img}} I_{u,v}\|_2,
\qquad
s_{u,v}=\|\nabla_{\mathrm{img}}\vphi_{u,v}\|_1.
\]
The reported score is the normalized weighted average
\[
\mathrm{EAS}(\vphi,I)=\frac{\sum_{(u,v)\in\Omega}s_{u,v}e_{u,v}}{\sum_{(u,v)\in\Omega}s_{u,v}+10^{-12}}.
\]
A higher score means that attribution-map edges are more aligned with image edges. Smooth maps can have high edge alignment if they concentrate changes near object boundaries, while noisy maps can have low edge alignment because their gradients are spread across flat regions.

\subsection{Deletion metrics}
\label{app:deletion_details}

Deletion evaluates whether removing highly attributed pixels reduces the target score. We sort pixels by attribution magnitude or signed positive attribution, depending on the method's output convention. We then replace the top-ranked pixels progressively with either a zero value or a blurred value. At each deletion fraction, we evaluate the target logit. The deletion score is the area under the resulting score curve. Lower is better because a faithful positive attribution should identify pixels whose removal quickly decreases the target score.

Deletion depends on the perturbation operator. Zero deletion can create unnatural black patches, while blur deletion can preserve low-frequency image statistics but still change the data distribution. We report both variants because agreement between them is more informative than either one alone.

\subsection{Completeness residual}
\label{app:completeness_details}

For a path attribution rule, completeness means that the attribution sum equals the score difference. In continuous time, the Aumann-Shapley integral satisfies completeness exactly by the chain rule. In discrete time, numerical quadrature creates a residual. For \(N\) evaluated examples, we report
\[
R_{\mathrm{eff}}=\frac{1}{N}\sum_{j=1}^N\left|\sum_{i=1}^d\hat{\Psi}_i^{(j)}-\bigl(f_c(\vx_1^{(j)})-f_c(\hat{\vx}_0^{(j)})\bigr)\right|.
\]
Here \(\hat{\Psi}_i^{(j)}\) is the discrete attribution for coordinate \(i\) on example \(j\), \(\vx_1^{(j)}\) is the observed endpoint, and \(\hat{\vx}_0^{(j)}\) is the learned reference endpoint obtained by backward integration. Table~\ref{tab:completeness} reports this residual for different integration step counts. Proposition~\ref{prop:discrete_efficiency} explains why the residual decreases when the path discretization becomes finer.

\subsection{Stability metrics}
\label{app:stability_metrics_details}

We measure stability across random seeds by comparing attribution maps generated by independently trained flows. Pixel variance is the average pointwise variance of normalized attribution maps. SSIM measures perceptual similarity between pairs of normalized maps. Rank correlation measures whether the feature ordering induced by attribution scores is stable. These metrics answer a different question from deletion. A method can be faithful but unstable if small changes to the path alter the feature ranking. Our path-stability experiments test the hypothesis that lower-action and lower-curvature paths reduce this variability.

\section{Additional results and visualizations}
\label{app:additional_results}

This section reports supplementary experiments and visualizations on numerical accuracy, transport geometry, and attribution behavior.

\subsection{Synthetic additive sanity check}
\label{app:additive_sanity_check}

We first check the additive case in
Proposition~\ref{prop:additive_scores}.
We sample 50,000 inputs in $d=10$ dimensions
with independent coordinates uniformly distributed
over $[-\pi,\pi]$ and use
\[
f(\vx)=\sum_{i=1}^d\sin(x_i),
\qquad \vx_0=\bm{0}.
\]
The exact coordinate contribution is $\sin(x_i)$.
We approximate the Aumann--Shapley integral along
the straight path from $\vx_0$ to $\vx$,
without using a learned flow.

We compare the left Riemann estimator in Eq.~(11)
with midpoint quadrature using
$K\in\{10,20,50,100,200\}$ integration intervals.
Both estimators use the same sampled inputs.

\begin{figure}[H]
\centering
\includegraphics[width=0.92\linewidth]
{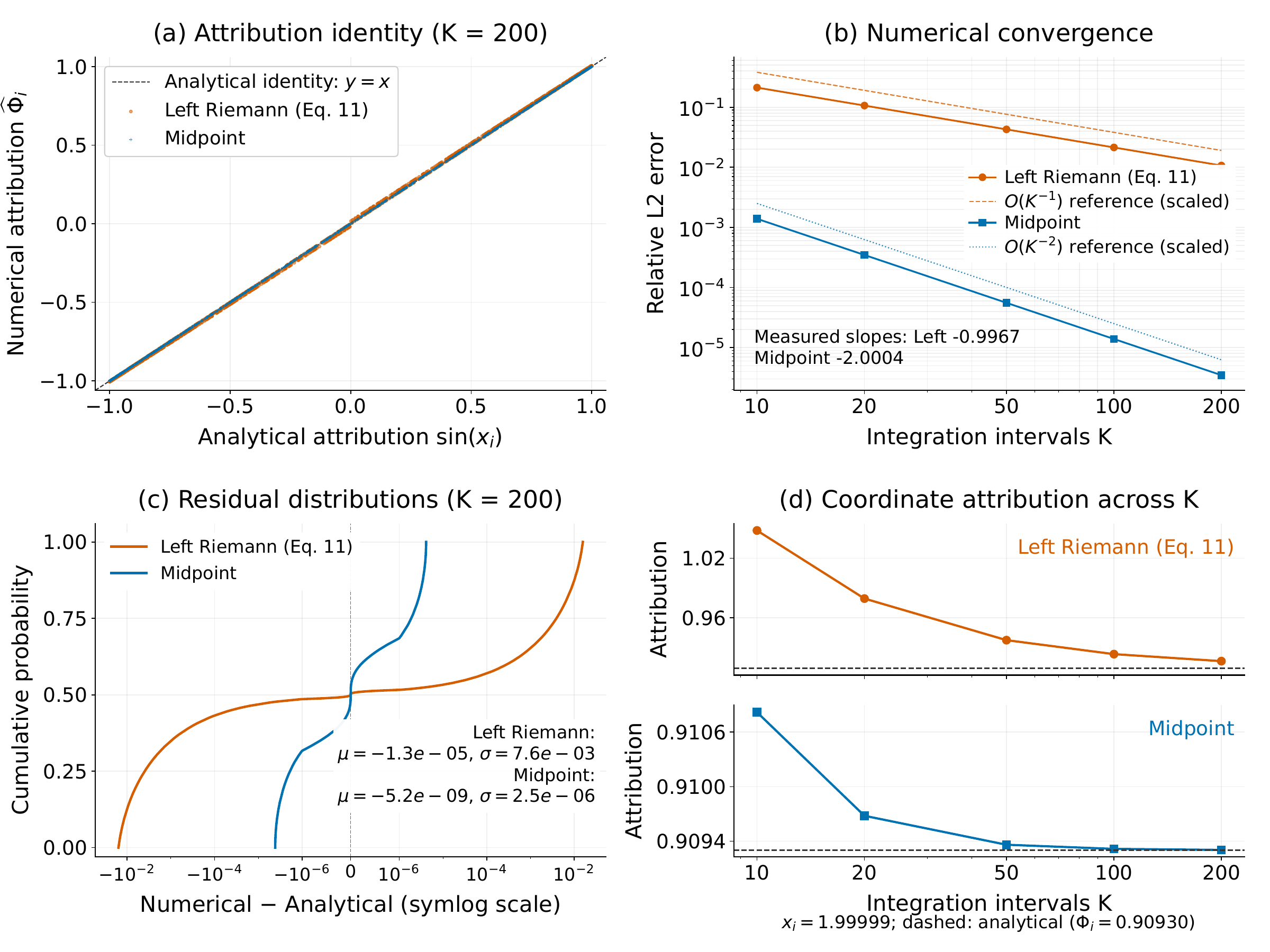}
\caption{Synthetic additive sanity check.
Both estimators converge to the analytical
additive contributions. The left Riemann
estimator exhibits first-order convergence,
while midpoint quadrature achieves
second-order convergence.}
\label{fig:additive_sanity_appendix}
\end{figure}

Figure~\ref{fig:additive_sanity_appendix}
shows near-identity alignment between analytical
and numerical attributions. The measured
convergence slopes are $-0.9967$ for left
Riemann and $-2.0004$ for midpoint quadrature,
consistent with their expected convergence rates.
These results numerically support the additive
allocation property for the tested straight paths,
consistent with Proposition~\ref{prop:additive_scores}.

\subsection{Controlled Gaussian transport experiment}
\label{app:gaussian_transport_experiment}

We next use Gaussian-to-Gaussian transport because the quadratic optimal transport map is available in closed form. This gives an oracle against which we can compare one-round Rectified Flow (1-RF) and Reflow. 

Let \(\cA^\star=W_2^2(p_0,p_1)\) denote the oracle kinetic action. For a learned method \(m\), let \(\widehat{\cA}_m\) denote the empirical discrete action averaged over \(N\) sampled trajectories,
\[
\widehat{\cA}_m=
\frac{1}{N}\sum_{j=1}^{N}\sum_{k=0}^{K-1}
\frac{\|\hat{\vx}_{k+1}^{m,(j)}-\hat{\vx}_{k}^{m,(j)}\|_2^2}{\Delta t}.
\]
The reported action gap is the relative excess action
\[
\Delta A_m=\frac{\widehat{\cA}_m-\cA^\star}{\cA^\star+10^{-12}}.
\]
The oracle column reports \(\Delta A=0\) using the analytical action \(\cA^\star\) as the reference. Let \(\vv_t^\star\) be the oracle velocity field and let \(\hat{\vv}_t^m\) be the learned velocity field for method \(m\). The relative field error is
\[
\mathrm{RFE}_m=
\frac{
\left(\sum_{k=0}^{K-1}\|\hat{\vv}_{t_k}^{m}(\hat{\vx}_{k}^{m})-\vv_{t_k}^{\star}(\hat{\vx}_{k}^{m})\|_2^2\right)^{1/2}
}{
\left(\sum_{k=0}^{K-1}\|\vv_{t_k}^{\star}(\hat{\vx}_{k}^{m})\|_2^2\right)^{1/2}+10^{-12}
}.
\]
The curvature proxy is the squared finite-difference acceleration integrated over the path,
\[
\mathrm{Curv}_m=
\sum_{k=1}^{K-1}
\left\|\frac{\hat{\vx}_{k+1}^{m}-2\hat{\vx}_{k}^{m}+\hat{\vx}_{k-1}^{m}}{\Delta t^2}\right\|_2^2\Delta t.
\]
These definitions make the table interpretable. \(\Delta A\) measures excess kinetic cost, RFE measures vector-field mismatch against the oracle, and Curv measures unnecessary bending of the sampled path.

\begin{figure}[H]
\centering
\includegraphics[width=0.95\linewidth]{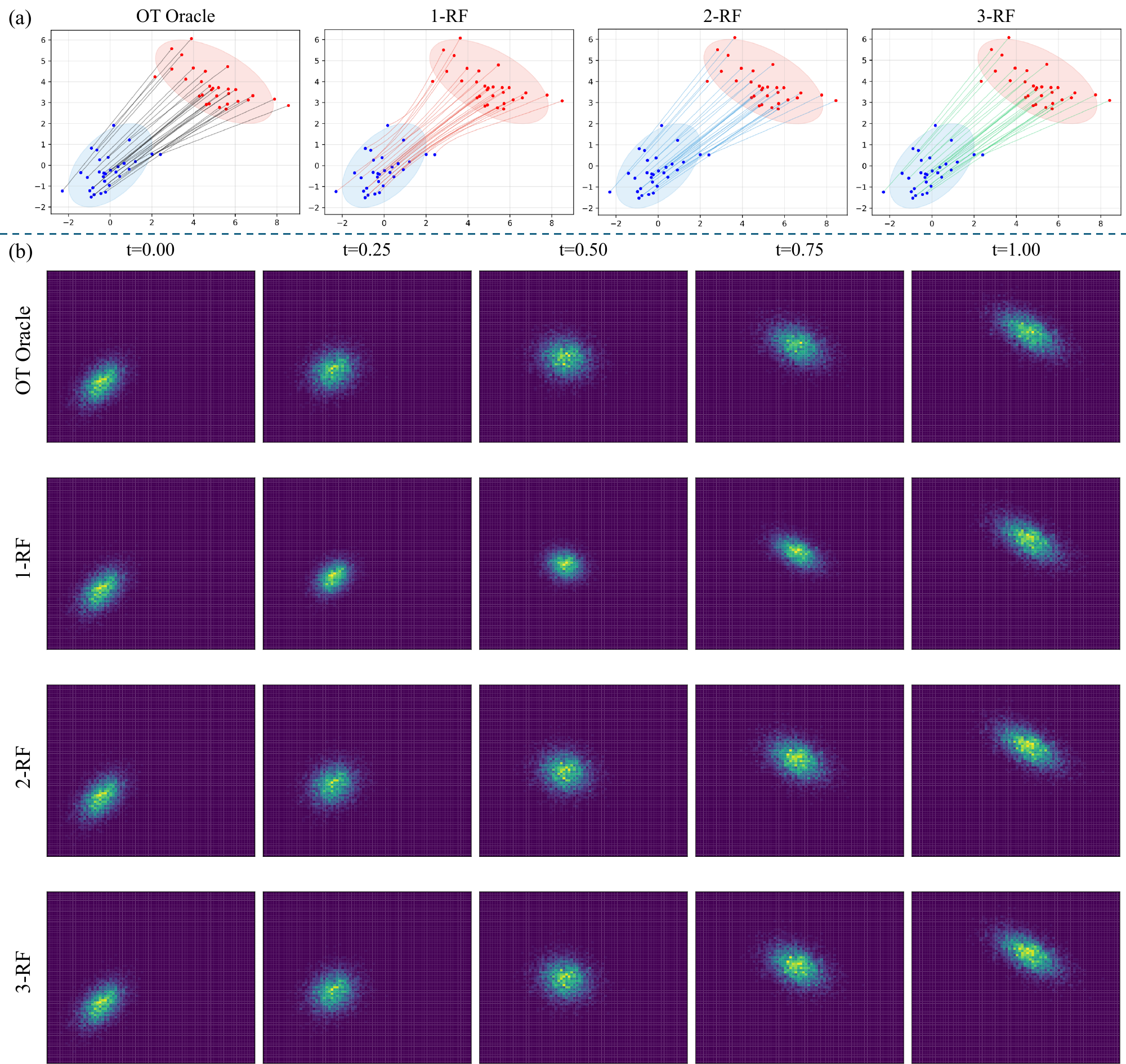}
\caption{Controlled two-dimensional Gaussian transport. The oracle displacement interpolation gives the reference geometry. One-round Rectified Flow can be visibly more curved, while Reflow straightens the trajectories and makes the intermediate densities closer to the oracle interpolation.}
\label{fig:toy_trajectories_density_appendix}
\end{figure}

The qualitative result in Figure~\ref{fig:toy_trajectories_density_appendix} explains why Reflow is useful for attribution. Since the attribution integral samples gradients along the path, unnecessary path curvature can expose the integral to unstable regions of the predictor. Reflow reduces this curvature in the controlled setting.

\begin{table}[H]
\centering
\small
\setlength{\tabcolsep}{5.0pt}
\resizebox{\linewidth}{!}{%
\begin{tabular}{lcccc}
\toprule
Metric & OT oracle & 1-RF & 2-RF & 3-RF \\
\midrule
\(W_2^2(p_0,p_1)\) & 27.8976 & - & - & - \\
Action gap \(\Delta A\) & 0 & \(0.2205\pm0.0033\) & \(0.00189\pm0.00144\) & \(0.00199\pm0.00141\) \\
Relative field error & 0 & \(0.2933\pm0.0050\) & \(0.00284\pm0.00014\) & \(0.00276\pm0.00008\) \\
Curvature proxy & \(3.45{\times}10^{-5}\pm2.63{\times}10^{-7}\) & \(0.0763\pm0.00077\) & \(4.85{\times}10^{-4}\pm1.50{\times}10^{-5}\) & \(4.35{\times}10^{-4}\pm1.31{\times}10^{-5}\) \\
\bottomrule
\end{tabular}}
\caption{Controlled Gaussian transport in \(d=10\), reported as mean \(\pm\) standard deviation over five seeds. Reflow reduces the relative action gap, relative field error, and curvature proxy compared with one-round Rectified Flow.}
\label{tab:toy_geom_d10_appendix}
\end{table}

Table~\ref{tab:toy_geom_d10_appendix} supports the same conclusion quantitatively. The one-round flow has a nontrivial action gap and field mismatch relative to the oracle. After one Reflow iteration, these gaps decrease by roughly two orders of magnitude in this controlled setting. This is the empirical basis for using Reflow as a closer approximation to the transport-geodesic ideal.

\subsection{Attribution error versus field mismatch}
\label{app:field_mismatch_results}

Theorem~\ref{thm:stability} predicts that attribution error should decrease when the learned field approaches the ideal field on the relevant region. We test this prediction in the Gaussian setting where the oracle is known, and also across flow checkpoints where a strong checkpoint serves as a practical oracle.

Let \(\bm{\Psi}^{\mathrm{ref}}\) be the attribution produced by the oracle or reference checkpoint, and let \(\hat{\bm{\Psi}}^m\) be the attribution produced by method or checkpoint \(m\). The relative attribution error is
\[
\mathrm{RAE}_m=
\frac{\|\hat{\bm{\Psi}}^m-\bm{\Psi}^{\mathrm{ref}}\|_2}{\|\bm{\Psi}^{\mathrm{ref}}\|_2+10^{-12}}.
\]
When an oracle vector field is available, the empirical field mismatch is the relative field error \(\mathrm{RFE}_m\) defined above. For checkpoint comparisons, we replace \(\vv_t^\star\) with the reference checkpoint's velocity field, yielding \(E_{\mathrm{field}}\) as defined in Section~\ref{sec:path_geometry_stability}. The Gaussian experiment therefore uses the known OT field, while the checkpoint experiment measures agreement with a learned reference.

\begin{figure}[H]
\centering
\includegraphics[width=0.95\linewidth]{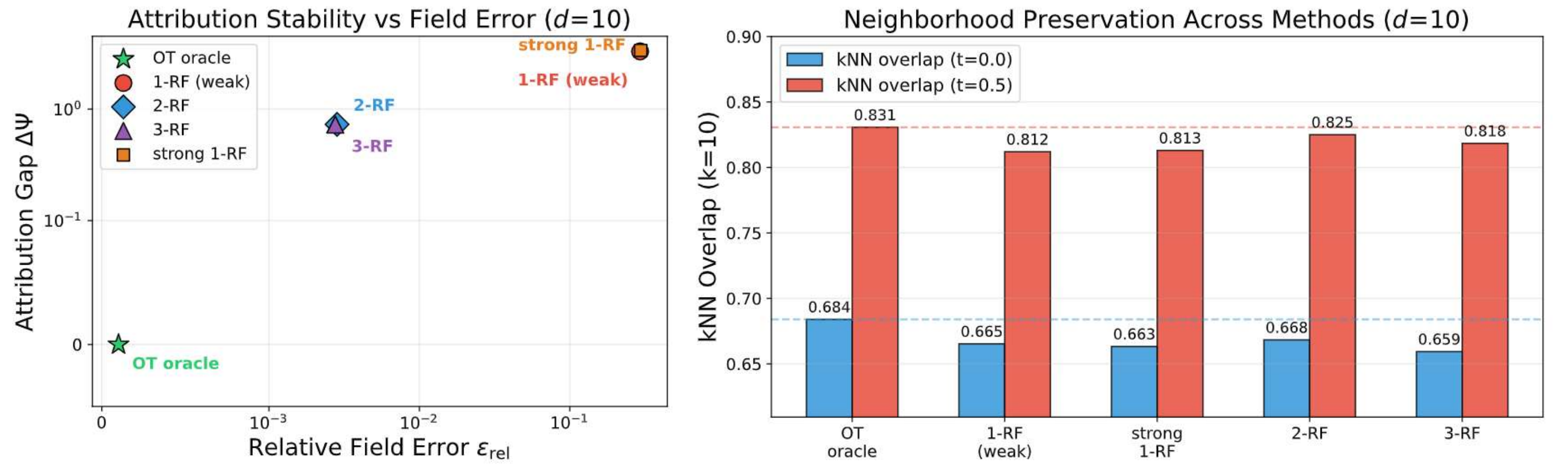}
\caption{High-dimensional controlled diagnostic. The left panel relates relative attribution error to relative field mismatch. The right panel reports neighborhood preservation along intermediate transport states. The result supports the stability view: better field alignment gives smaller attribution discrepancy.}
\label{fig:d10_e3_appendix}
\end{figure}

\begin{figure}[H]
\centering
\includegraphics[width=0.56\linewidth]{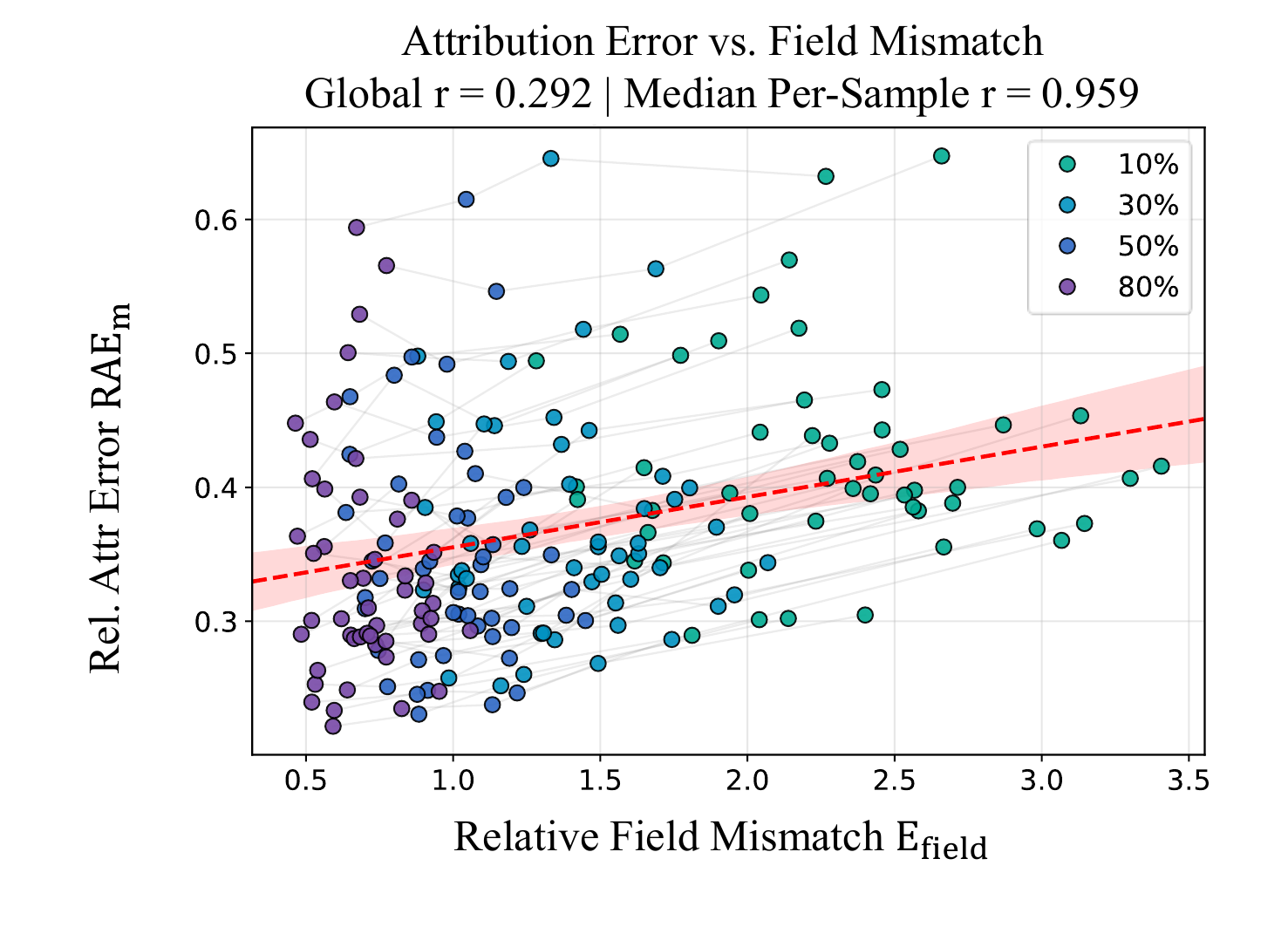}
\caption{Attribution convergence across flow checkpoints. The horizontal axis measures normalized field mismatch \(E_{\mathrm{field}}\), while the vertical axis measures relative attribution error \(\mathrm{RAE}\), both evaluated against the reference checkpoint.}
\label{fig:convergence_appendix}
\end{figure}

Figures~\ref{fig:d10_e3_appendix} and~\ref{fig:convergence_appendix} are not meant to show a universal linear law. They show that, in our controlled diagnostics, attribution error tracks vector-field mismatch in the direction predicted by Theorem~\ref{thm:stability}.

\subsection{Action diagnostics}
\label{app:action_diagnostics}

We also visualize representative action estimates along sampled paths. Lower action does not automatically imply better explanation under every metric, but high unnecessary action often means that the path moves more than needed before reaching the same endpoint. Such movement can make the line integral more sensitive to local gradient variation.

\begin{figure}[H]
\centering
\includegraphics[width=0.54\linewidth]{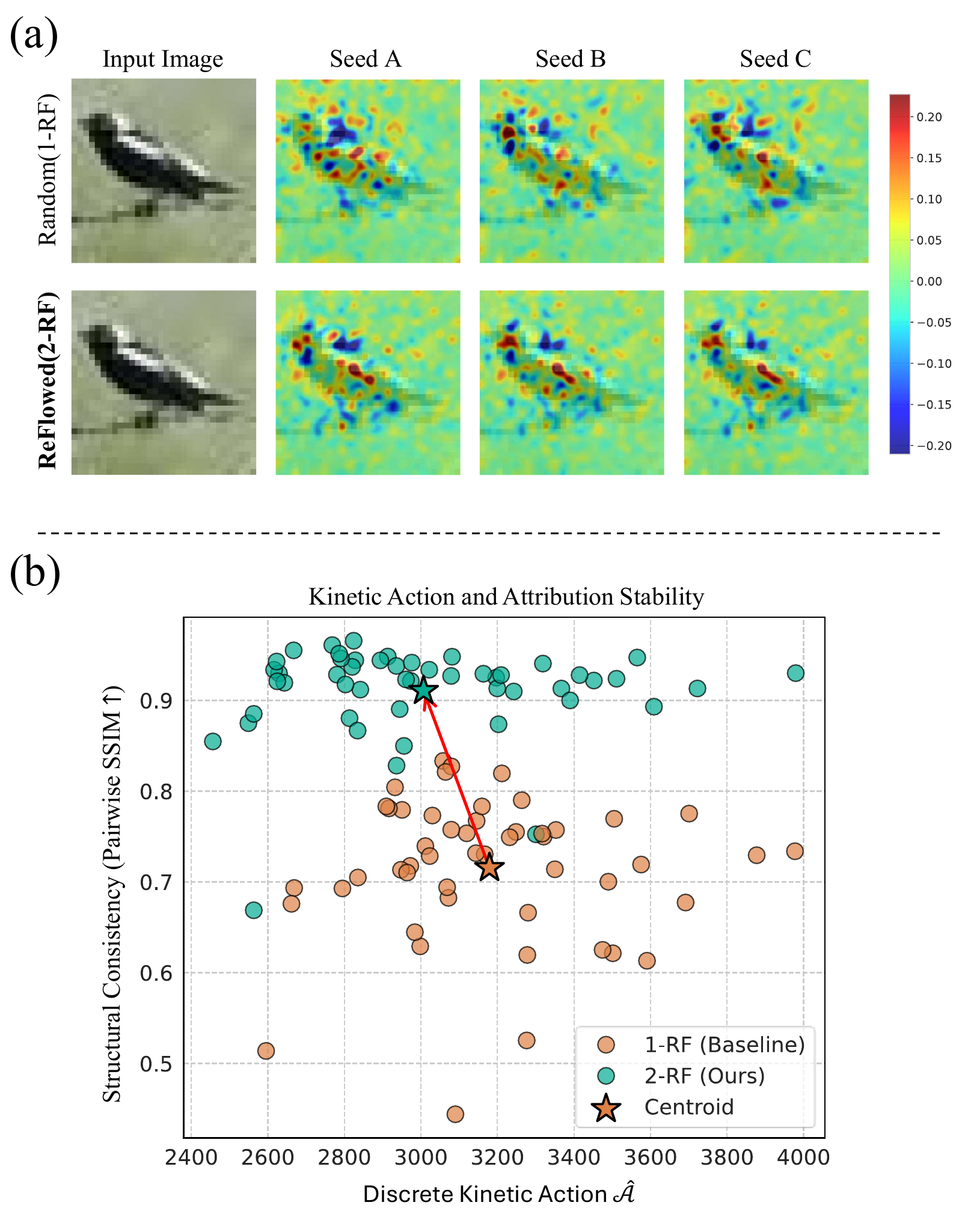}
\caption{Kinetic action and attribution stability across flow models. (a) Representative attribution maps from three random seeds for 1-RF and 2-RF. (b) Relationship between discrete kinetic action and pairwise attribution SSIM, with stars indicating group centroids. Reflow is associated with lower action and higher attribution stability.}
\label{fig:action_diagnostic_appendix}
\end{figure}

% \paragraph{Reflow and attribution faithfulness.}
% We compare 1-RF with 2-RF on a subset of CIFAR-10.
% The latter applies one additional Reflow stage using
% endpoint pairs generated by 1-RF.
% Table~\ref{tab:reflow_faithfulness} shows improvements
% in both deletion metrics and insertion, complementing
% the action and stability diagnostics.

% \begin{table}[t]
% \centering
% \caption{Attribution faithfulness before and after one
% additional Reflow stage, evaluated on a CIFAR-10 subset.}
% \label{tab:reflow_faithfulness}
% \small
% \begin{tabular}{lccc}
% \toprule
% Method
% & Del.\ black \(\downarrow\)
% & Del.\ blur \(\downarrow\)
% & Insertion \(\uparrow\) \\
% \midrule
% 1-RF & 0.1795 & 0.3134 & 0.4303 \\
% 2-RF & \textbf{0.1711} & \textbf{0.3069} & \textbf{0.4336} \\
% \bottomrule
% \end{tabular}
% \end{table}

\subsection{Same-endpoint path ablation}
\label{app:same_endpoint_ablation}
To distinguish the contribution of reference
selection from that of the learned trajectory,
we compare the 1-RF characteristic with
straight-line Integrated Gradients (IG) using
the same recovered endpoint pair
$(\hat{x}_0,x_1)$.
Both methods use the same number of integration
steps and identical attribution and evaluation
protocols. We additionally include conventional
IG with its original reference choice.
The experiment is conducted on a subset of
CIFAR-10.

\begin{table}[h]
%\label{tab:same_endpoint_path_ablation}
\centering
\caption{Same-endpoint path ablation on a CIFAR-10
subset. Straight IG (same endpoint) and 1-RF
use the same recovered reference $\hat{x}_0$
and target input $x_1$.}
\label{tab:same_endpoint_ablation}
\small
\setlength{\tabcolsep}{5pt}
\begin{tabular}{lccc}
\toprule
Method &
Del.-Black $\downarrow$ &
Del.-Blur $\downarrow$ &
Insertion $\uparrow$ \\
\midrule
IG (original reference)
    & 0.1825 & 0.3866 & 0.4179 \\
1-RF path
    & 0.1795 & 0.3134 & \textbf{0.4303} \\
Straight IG (same endpoint)
    & \textbf{0.1727}
    & \textbf{0.3010}
    & 0.4286 \\
\bottomrule
\end{tabular}
\end{table}

Table~\ref{tab:same_endpoint_ablation}
allows us to examine reference selection and
trajectory effects separately.
Compared with conventional IG, straight-line IG
using the 1-RF-recovered reference achieves
lower deletion scores and a higher insertion
score on the evaluated CIFAR-10 subset.
This suggests that the learned reference
selection contributes to attribution faithfulness.

With the reference endpoints fixed, straight-line
IG and the 1-RF characteristic achieve comparable
faithfulness. Straight-line IG yields slightly
lower deletion scores, whereas the 1-RF
characteristic achieves a slightly higher
insertion score. These results indicate no
consistent advantage from the residual trajectory
curvature of the learned flow.

This observation is consistent with the ideal
quadratic optimal transport formulation.
For fixed endpoints, the constant-speed straight
segment minimizes kinetic action.
Under exact quadratic optimal transport,
the optimal coupling determines the reference
endpoint, and the corresponding characteristic
follows straight displacement interpolation.
Consequently, ideal transport-geodesic attribution
coincides with IG using the OT-selected reference.

In practice, the learned Rectified Flow provides
a tractable transport coupling and its associated
characteristic trajectories without guaranteeing
exact optimality. Differences between the learned
characteristic and same-endpoint IG therefore
reflect residual trajectory geometry and
numerical approximation effects.

\subsection{Additional qualitative examples}
\label{app:additional_qualitative}

The following figures show additional randomly selected
attribution examples on CIFAR-10 and CelebA-HQ, illustrating the spatial structure of the resulting maps.

\begin{figure}[h]
\centering
\includegraphics[width=\linewidth]{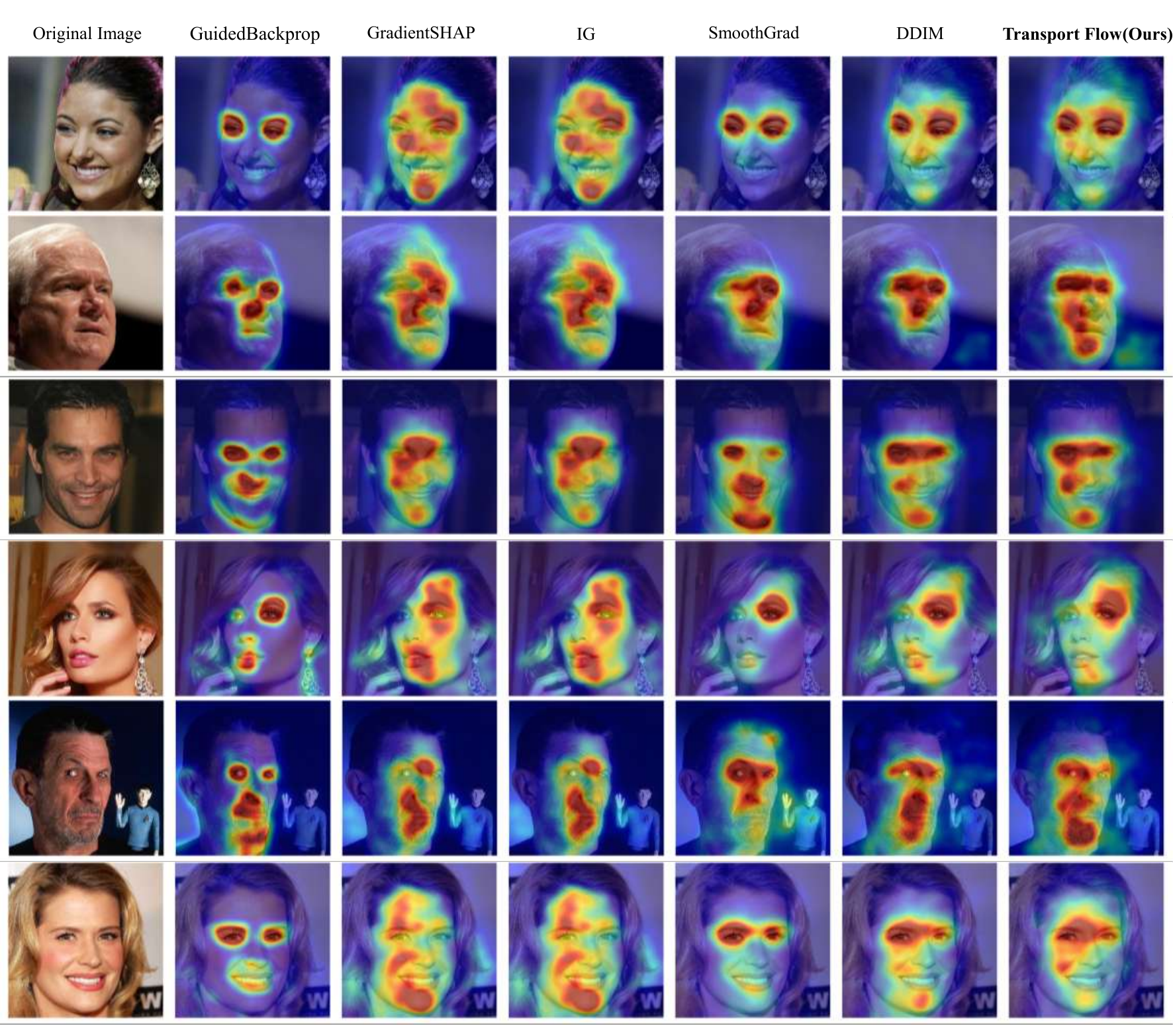}
\caption{Additional CelebA-HQ examples, part I.}
\label{fig:celeba_appendix_1}
\end{figure}

\begin{figure}[p]
\centering
\includegraphics[width=\linewidth]{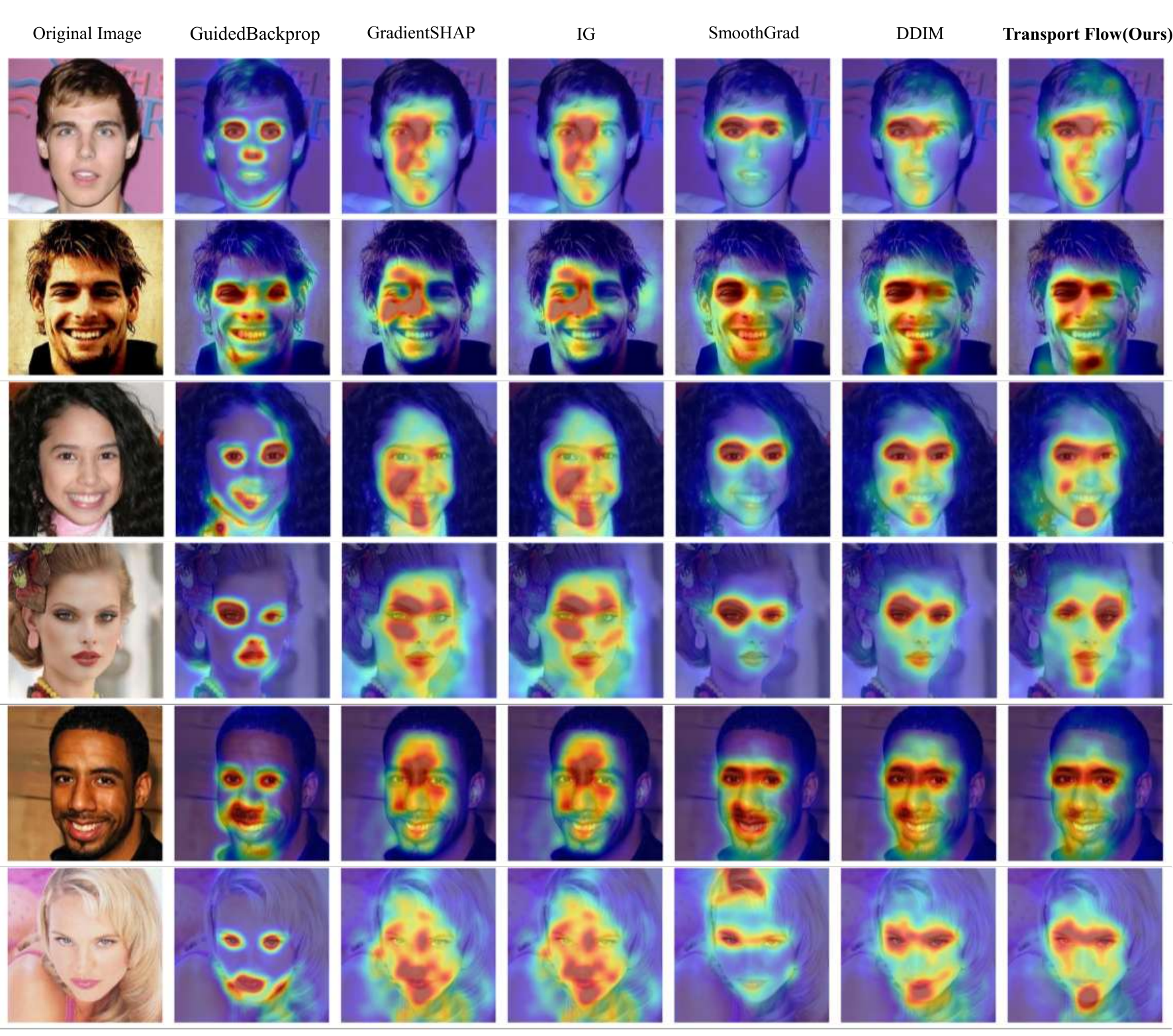}
\caption{Additional CelebA-HQ examples, part II.}
\label{fig:celeba_appendix_2}
\end{figure}

\begin{figure}[p]
\centering
\includegraphics[width=\linewidth]{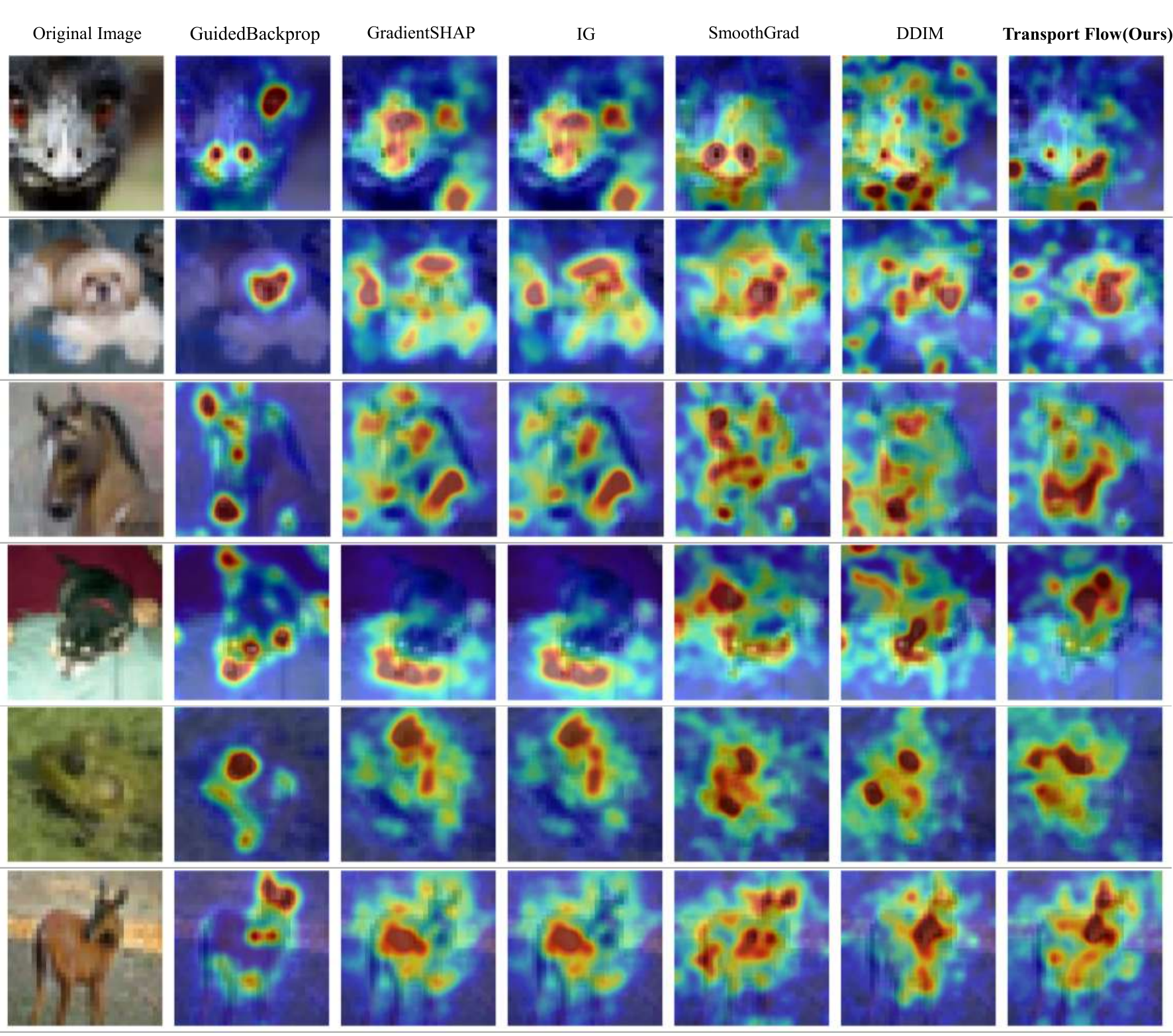}
\caption{Additional CIFAR-10 examples, part I.}
\label{fig:cifar_appendix_1}
\end{figure}

\begin{figure}[p]
\centering
\includegraphics[width=\linewidth]{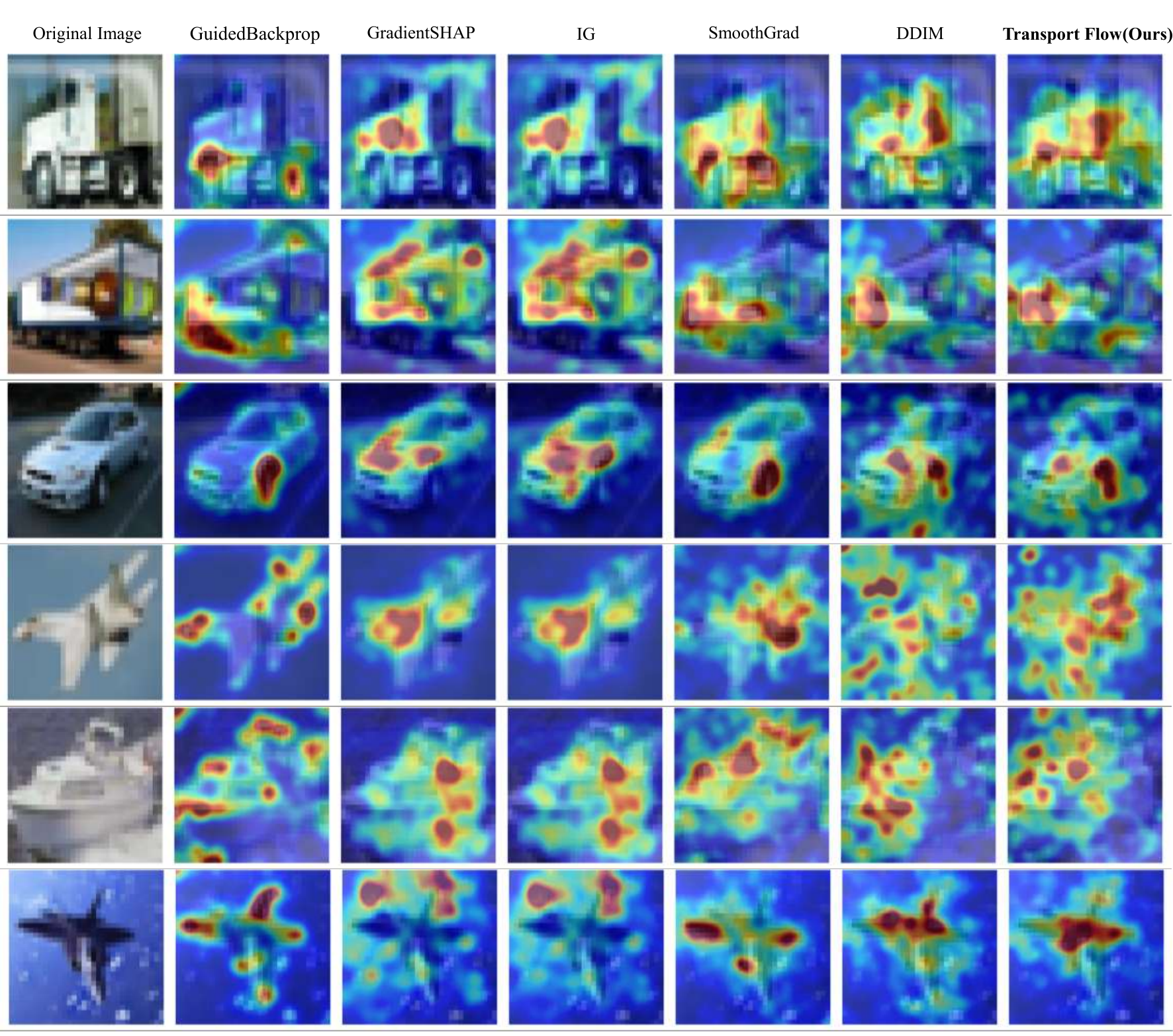}
\caption{Additional CIFAR-10 examples, part II.}
\label{fig:cifar_appendix_2}
\end{figure}

\clearpage
\section{Limitations and scope of the claims}
\label{app:limitations}

This appendix section states what the paper does not claim. This section discusses limitations arising from learned transport, computational cost, and evaluation.

\subsection{No strict manifold certificate}
\label{app:limitation_no_manifold}

We do not prove that the learned path stays on a true data manifold. A strict manifold statement would require a defined manifold, a membership criterion, and a proof that the learned ODE trajectory satisfies that criterion for every time. Our theory instead uses a path of probability distributions that transports \(p_0\) to \(p_1\). The intermediate laws are transport marginals. This is enough for the attribution principle, because the principle needs a selected counterfactual transition, not a manifold certificate.

% This limitation changes how one should read FCE. FCE close to zero means that a numerical path follows the learned vector field. It does not mean that the true data density is high at every point. We therefore use FCE as a dynamics diagnostic and not as evidence of strict on-manifold behavior.

\subsection{Dependence on the learned generator}
\label{app:limitation_generator}

The method inherits errors from the learned flow. If \(\vv_{\vtheta}\) fails to approximate the target transport field near the trajectory for \(\vx_1\), then the attribution may follow a poor path. Theorem~\ref{thm:stability} explains this dependence by bounding attribution error in terms of vector-field error under regularity assumptions. The theorem does not remove the need for a good generator. It only tells us how generator error propagates once the assumptions hold.

This also means that datasets with poor generative coverage or strong spurious correlations can produce misleading transport paths. In such cases, the attribution may reflect the generator's learned biases together with the classifier's behavior. We recommend checking generator quality and using the method together with sanity checks rather than treating it as a standalone certificate.

\subsection{Approximate optimal transport}
\label{app:limitation_approx_ot}

The ideal path is the kinetic-action minimizer. In high-dimensional image experiments, we do not solve the exact Benamou-Brenier problem. Rectified Flow and Reflow give a scalable approximation. The controlled Gaussian experiments show that Reflow reduces the relative action gap and vector-field mismatch against the OT oracle. This provides empirical evidence of improved transport approximation in the evaluated setting.

The uniqueness language in the main theorem should also be read relative to the ideal problem. If the OT dynamic plan is not unique, then the path-selection problem can have multiple minimizers. Assumption~\ref{ass:transport_regular} rules out that ambiguity for the formal theorem. In applications, different trained flows can approximate different near-minimizing paths, which is why we report stability across seeds.

\subsection{Computational cost}
\label{app:limitation_cost}

The method is more expensive than one backward pass. It needs ODE integration and repeated gradients along the path. This is still linear in the number of integration steps and much cheaper than exact feature-level Shapley values, but it is slower than SmoothGrad with a small number of noise samples or a single Grad-CAM pass. The cost is most relevant for high-resolution images and large predictors.

Several engineering choices can reduce the cost. We can use fewer path steps when approximate completeness is acceptable, use midpoint or adaptive quadrature to improve accuracy per step, cache path states, or group pixels into superpixels before deletion evaluation. These choices change the numerical estimator but not the continuous attribution principle.

\subsection{Evaluation limitations}
\label{app:limitation_evaluation}

Deletion, blur deletion, insertion, SATV, EAS, GPS, and FCE
each measure only one aspect of explanation behavior.
Deletion and insertion depend on the perturbation operator.
SATV can reward smooth maps even when they are not faithful.
EAS measures visual alignment with image edges but not
causal relevance. GPS measures path straightness but not
distributional correctness.
FCE is evaluated against the same learned field used
to generate Transport Flow paths.
It measures dynamical consistency rather than
independent attribution faithfulness, density support,
or transport optimality.

The synthetic additive sanity check in
Appendix~\ref{app:additive_sanity_check}
assesses numerical recovery of analytical coordinate
contributions by the fixed-path estimator, separately
from learned path selection.
The reported Reflow comparisons associate lower action
with greater attribution stability, while the image
benchmarks provide evidence of structured maps and
competitive deletion faithfulness.
The empirical conclusions concern the evaluated
datasets, predictors, and perturbation protocols.

% \subsection{Extensions to other modalities}
% \label{app:extensions_modalities}

% Although our experiments focus on vision, the fixed-path attribution result applies more broadly to differentiable scalar predictors along sufficiently regular paths. Extending the transport component additionally requires meaningful reference and target distributions and a continuous flow. Continuous tabular data is a natural next step, while categorical features may require constrained flows or continuous relaxations. For NLP, a possible extension is to operate in embedding space and map the resulting attributions back to tokens. We leave empirical validation of these extensions to future work.

\section*{NeurIPS Paper Checklist}

\begin{enumerate}

\item {\bf Claims}
    \item[] Question: Do the main claims made in the abstract and introduction accurately reflect the paper's contributions and scope?
    \item[] Answer: \answerYes{}
    \item[] Justification: The abstract and introduction state the intended scope: the paper studies path ambiguity, separates fixed-path credit allocation from path selection, and proposes a transport-geodesic path-selection principle. The claims are intentionally limited to stability, path geometry, and competitive faithfulness rather than universal superiority.

\item {\bf Limitations}
    \item[] Question: Does the paper discuss the limitations of the work performed by the authors?
    \item[] Answer: \answerYes{}
    \item[] Justification: The paper discusses limitations in the conclusion and in Appendix~\ref{app:limitations}, including no strict data-manifold certificate, dependence on the learned generator, approximate optimal transport, computational cost, and limits of attribution metrics.

\item {\bf Theory assumptions and proofs}
    \item[] Question: For each theoretical result, does the paper provide the full set of assumptions and a complete (and correct) proof?
    \item[] Answer: \answerYes{}
    \item[] Justification: The main text states the assumptions for the fixed-path uniqueness theorem and the stability result, including coordinate-trace richness and regularity of the transport flow. Complete proofs are provided in Appendix~\ref{app:proofs}, with auxiliary transport details in Appendix~\ref{app:transport_details}.

    \item {\bf Experimental result reproducibility}
    \item[] Question: Does the paper fully disclose all the information needed to reproduce the main experimental results of the paper to the extent that it affects the main claims and/or conclusions of the paper (regardless of whether the code and data are provided or not)?
    \item[] Answer: \answerNo{}
    \item[] Justification: The paper gives the code, algorithm, objectives, datasets, metrics, and evaluation protocol in abstract, Section~\ref{sec:experiments} and Appendices~\ref{app:implementation_details}--\ref{app:metrics_details}. However, the submission does not yet include trained checkpoints, and exact low-level implementation settings needed for exact numerical reproduction.

\item {\bf Open access to data and code}
    \item[] Question: Does the paper provide open access to the data and code, with sufficient instructions to faithfully reproduce the main experimental results, as described in supplemental material?
    \item[] Answer: \answerYes{}
    \item[] Justification: The experiments use public benchmark datasets and cite the corresponding sources, 
    and the current submission provide an open anonymous code. 
    The method is described algorithmically, 
    and the authors plan to release the last full-specified code after review if permitted by the submission policy.

\item {\bf Experimental setting/details}
    \item[] Question: Does the paper specify all the training and test details (e.g., data splits, hyperparameters, how they were chosen, type of optimizer) necessary to understand the results?
    \item[] Answer: \answerNo{}
    \item[] Justification: The paper specifies the core estimator, Rectified Flow/Reflow objectives, datasets, baselines, path-step notation, and metrics. It does not yet provide every exact engineering detail such as all network architectures, optimizer settings, hardware-dependent solver settings, and full random seed configuration.

\item {\bf Experiment statistical significance}
    \item[] Question: Does the paper report error bars suitably and correctly defined or other appropriate information about the statistical significance of the experiments?
    \item[] Answer: \answerYes{}
    \item[] Justification: The quantitative tables report mean values with standard deviations for the benchmark metrics and for controlled Gaussian transport diagnostics. The stability experiments also compare independently trained flows across seeds.

\item {\bf Experiments compute resources}
    \item[] Question: For each experiment, does the paper provide sufficient information on the computer resources (type of compute workers, memory, time of execution) needed to reproduce the experiments?
    \item[] Answer: \answerNo{}
    \item[] Justification: Appendix~\ref{app:cost_details} analyzes algorithmic cost and explains that one explanation scales linearly with the number of path steps. The paper does not yet report the exact hardware type, memory, wall-clock time, or total compute used for each experimental run.
    
\item {\bf Code of ethics}
    \item[] Question: Does the research conducted in the paper conform, in every respect, with the NeurIPS Code of Ethics \url{https://neurips.cc/public/EthicsGuidelines}?
    \item[] Answer: \answerYes{}
    \item[] Justification: The work is a methodological study of model explanations using standard public image benchmarks and does not involve deception, intervention on users, or sensitive deployment. The paper also includes an impact statement and discusses limitations and responsible use.

\item {\bf Broader impacts}
    \item[] Question: Does the paper discuss both potential positive societal impacts and negative societal impacts of the work performed?
    \item[] Answer: \answerYes{}
    \item[] Justification: The impact statement discusses potential positive uses for model auditing and scientific or medical analysis, as well as risks such as overinterpreting attributions as causal guarantees and inheriting biases from learned generative models.
    
\item {\bf Safeguards}
    \item[] Question: Does the paper describe safeguards that have been put in place for responsible release of data or models that have a high risk for misuse (e.g., pre-trained language models, image generators, or scraped datasets)?
    \item[] Answer: \answerNA{}
    \item[] Justification: The paper does not release a high-risk pretrained model, scraped dataset, or dual-use generative system. It studies an attribution method evaluated on existing public benchmarks, so special release safeguards are not applicable.

\item {\bf Licenses for existing assets}
    \item[] Question: Are the creators or original owners of assets (e.g., code, data, models), used in the paper, properly credited and are the license and terms of use explicitly mentioned and properly respected?
    \item[] Answer: \answerNo{}
    \item[] Justification: The paper credits the existing datasets and baseline methods through citations, including CUB-200, CIFAR-10, and CelebA-HQ. The current manuscript does not yet explicitly list the licenses or terms of use for each external asset.

\item {\bf New assets}
    \item[] Question: Are new assets introduced in the paper well documented and is the documentation provided alongside the assets?
    \item[] Answer: \answerNA{}
    \item[] Justification: The paper does not introduce a new dataset, benchmark, pretrained model release, or other standalone asset. It introduces a method and reports experiments on existing assets.

\item {\bf Crowdsourcing and research with human subjects}
    \item[] Question: For crowdsourcing experiments and research with human subjects, does the paper include the full text of instructions given to participants and screenshots, if applicable, as well as details about compensation (if any)? 
    \item[] Answer: \answerNA{}
    \item[] Justification: The paper does not involve crowdsourcing, user studies, or new research with human participants. It uses existing public image datasets and does not collect new human-subject data.

\item {\bf Institutional review board (IRB) approvals or equivalent for research with human subjects}
    \item[] Question: Does the paper describe potential risks incurred by study participants, whether such risks were disclosed to the subjects, and whether Institutional Review Board (IRB) approvals (or an equivalent approval/review based on the requirements of your country or institution) were obtained?
    \item[] Answer: \answerNA{}
    \item[] Justification: The paper does not conduct new human-subject experiments, crowdsourcing studies, or interventions with participants. Therefore IRB or equivalent approval is not applicable to the reported work.

\item {\bf Declaration of LLM usage}
    \item[] Question: Does the paper describe the usage of LLMs if it is an important, original, or non-standard component of the core methods in this research? Note that if the LLM is used only for writing, editing, or formatting purposes and does \emph{not} impact the core methodology, scientific rigor, or originality of the research, declaration is not required.
    %this research? 
    \item[] Answer: \answerNA{}
    \item[] Justification: LLMs are not used as an important, original, or non-standard component of the proposed method, theory, or experiments. Any language-model assistance, if used for writing or formatting, does not affect the core methodology or scientific claims.
\end{enumerate}

\end{document}